\documentclass[11pt,letterpaper]{article}
\usepackage[affil-it]{authblk}

\usepackage{chihao}
\usepackage[margin=1in]{geometry}

\usepackage{varwidth}

\usepackage{graphicx}
\usepackage{threeparttable}
\usepackage{multirow}
\usepackage{array}
\usepackage{makecell}
\usepackage{physics}
\usepackage{xifthen}
\usepackage{booktabs}
\usepackage{multirow}
\usepackage{array}
\usepackage{makecell}
\allowdisplaybreaks[4] 
\usepackage{tikz}
\usepackage{rotating}
\usetikzlibrary{math}
\usetikzlibrary{positioning}
\usepackage[justification = centering, labelsep =period]{caption}
\LinesNumbered
\SetKwInOut{Input}{Input}
\SetKwInOut{Output}{Output}
\newcommand{\DKL}{D_{\!{KL}}\xspace}
\newcommand{\PP}[2][]{ \ifthenelse{\isempty{#1}}
  {\mathbf{P}\left[#2\right]} {\mathbf{P}_{#1}\left[#2\right]} }

\newcommand{\BAI}{\texttt{BAI}\xspace}
\newcommand{\MAB}{\texttt{MAB}\xspace}
\newcommand{\PAC}{\texttt{PAC}\xspace}

\newcommand{\Nout}{N_{\!{out}}}
\newcommand{\Nin}{N_{\!{in}}}
\newcommand{\mgne}{\succ}

\DeclareMathOperator{\dom}{\!{dom}}

\DeclareMathOperator{\rect}{\!{Rect}}

\title{On Interpolating Experts and Multi-Armed Bandits}
\author{Houshuang Chen\thanks{chenhoushuang@sjtu.edu.cn}}
\author{Yuchen He\thanks{yuchen\_he@sjtu.edu.cn}}
\author{Chihao Zhang\thanks{chihao@sjtu.edu.cn}}
\affil{Shanghai Jiao Tong University}
\date{}

\begin{document}

\maketitle

\begin{abstract}
	Learning with expert advice and multi-armed bandit are two classic online decision problems which differ on how the information is observed in each round of the game. We study a family of problems interpolating the two. For a vector $\*m=(m_1,\dots,m_K)\in \bb N^K$, an instance of $\*m$-\MAB indicates that the arms are partitioned into $K$ groups and the $i$-th group contains $m_i$ arms. Once an arm is pulled, the losses of all arms in the same group are observed. 

	We prove tight minimax regret bounds for $\*m$-\MAB and design an optimal PAC algorithm for its pure exploration version, $\*m$-\BAI, where the goal is to identify the arm with minimum loss with as few rounds as possible. We show that the minimax regret of $\*m$-\MAB is $\Theta\tuple{\sqrt{T\sum_{k=1}^K\log (m_k+1)}}$ and the minimum number of pulls for an $(\eps,0.05)$-PAC algorithm of $\*m$-\BAI is $\Theta\tp{\frac{1}{\eps^2}\cdot \sum_{k=1}^K\log (m_k+1)}$. 
	
	Both our upper bounds and lower bounds for $\*m$-\MAB can be extended to a more general setting, namely the bandit with graph feedback, in terms of the \emph{clique cover} and related graph parameters. As consequences, we obtained tight minimax regret bounds for several families of feedback graphs.




	
\end{abstract}

\section{Introduction}\label{sec: intro}

A typical family of online decision problems is as follows: 
In each round of the game, the player chooses one of $N$ arms to pull. At the same time, the player will incur a loss of the pulled arm. The objective is to minimize the expected regret defined as the difference between the cumulative losses of the player and that of the single best arm over $T$ rounds. The minimax regret, denoted as $R^*(T)$, represents the minimum expected regret achievable by any algorithm against the worst loss sequence. 

There are variants of the problem according to amount of information the player can observe in each round. In the problem of multi-armed bandit (\MAB), the player can only observe the loss of the arm just pulled. The minimax regret is $\Theta\tp{\sqrt{NT}}$~(\cite{AB09}). Another important problem is when the player can observe the losses of all arms in each round, often refered to as learning with expert advice. The minimax regret is $\Theta\tp{\sqrt{T\log N}}$~(\cite{FS97,HKW95}). Bandit with graph feedback generalizes and interpolates both models. In this model, a directed graph $G$, called the feedback graph, is given. The vertex set of $G$ is the set of arms and a directed edge from $i$ to $j$ indicates that pulling the arm $i$ can observe the loss of arm $j$. As a result, the \MAB corresponds to when $G$ consists of singletons with self-loop, and learning with expert advice corresponds to when $G$ is a clique. A number of recent works devote to understanding how the structure of $G$ affects the minimax regret~(\cite{ACDK15, CHLZ21, HZ22, EECCB23, KC23, RVCS22, DWZ23}). 

In this paper, we consider a natural interpolation between learning with expert advice and multi-armed bandit. Let $\*m=(m_1,m_2,\dots,m_K)\in\bb N^K$ be a vector with each $m_i\ge 1$. An instance of $\*m$-\MAB is that the all $N$ arms are partitioned into $K$ groups and the pull of each arm can observe the losses of all arms in the same group. In the language of bandit with graph feedback, the feedback graph $G$ is the disjoint union of $K$ cliques with size $m_1,m_2,\dots,m_k$ respectively. We show that the minimax regret for $\*m$-\MAB is $\Theta\tp{\sqrt{T\cdot \sum_{k\in[K]} \log (m_k+1)}}$. As a result, this generalizes the optimal regret bounds for both \MAB and learning with expert advice. 

A closely related problem is the so-called ``pure exploration'' version of bandit, often referred to as the \emph{best arm identification }(\BAI) problem where the loss of each arm follows some (unknown) distribution. The goal of the problem is to identify the arm with minimum mean loss with as few rounds as possible. Similarly, we introduced the problem of $\*m$-\BAI with the same feedback pattern as $\*m$-\MAB. We design an $(\eps,0.05)$-PAC algorithm for $\*m$-BAI which terminates in $T=O\tp{\frac{1}{\eps^2}\sum_{k\in [K]}\log (m_k+1)}$ rounds for every $\eps<\frac{1}{8}$. This means that after $T$ rounds of the game, with probability at least $0.95$, the algorithm can output an arm whose mean loss is less than $\eps$ plus the mean of the best one. We show that our algorithm is optimal by proving a matching lower bound $\Omega\tp{\frac{1}{\eps^2}\sum_{k\in [K]}\log (m_k+1)}$ for any $(\eps,0.05)$-PAC algorithm.

Both our upper bounds and lower bounds for the minimax regret of $\*m$-\MAB can be generalized to bandit with graph feedback. To capture the underlying structure necessary for our proofs, we introduce some new graph parameters which yield optimal bound for several families of feedback graphs. The main results are summarized in \Cref{sec: main results}.

Our algorithm deviates from the standard \emph{online stochastic mirror descent }(OSMD) algorithm for bandit problems. We employ the two-stage OSMD developed in~\cite{HZ22} and give a novel analysis which yields the optimal regret bound. For the lower bound, we prove certain new ``instance-specific'' lower bounds for the best arm identification problem. These lower bounds may find applications in other problems. We will give an overview of our techniques in \Cref{sec: tech}.

\subsection{Main Results}\label{sec: main results}

We summarize our main results in this section. Formal definitions of $\*m$-\MAB, $\*m$-\BAI and bandit with graph feedback are in \Cref{sec: prelim}. 


\begin{theorem} \label{thm:ub-m-MAB}
	There exists an algorithm such that for any instance of $(m_1,\dots,m_K)$-\MAB, any $T>0$ and any loss sequence $\ell^{(0)},\ell^{(1)},\dots,\ell^{(T-1)} \in [0,1]^N$, its regret is at most
	\[
	 c\cdot \sqrt{T\cdot \sum_{k=1}^K\log (m_k+1)},
	\]
	where $c>0$ is a universal constant.
\end{theorem}

Given an instance of $\*m$-\BAI, for $\eps,\delta \in (0,1)$, an $(\eps,\delta)$-PAC algorithm can output an arm whose mean loss is less than $\eps$ plus the mean of the optimal one with probability at least $1-\delta$. Using a reduction from $\*m$-\BAI to $\*m$-\MAB (\Cref{lem: reduction}), we obtain a PAC algorithm for $\*m$-\BAI:

\begin{theorem}\label{thm:ub-m-BAI}
	There exists an $(\eps,0.05)$-PAC algorithm for $(m_1,\dots,m_K)$-\BAI which pulls  
	\[
		T\le c\cdot \sum_{k=1}^K\frac{\log (m_k+1)}{\eps^2}
	\]
	arms where $c>0$ is a universal constant.
\end{theorem}


Let $\!{Ber}(p)$ denote the Bernoulli distribution with mean $p$. We complement the above algorithm with the following lower bound:

\begin{theorem}\label{thm:lb-m-BAI}
	There exists an instance $\@H$ such that for every $\tp{\eps, 0.05}$-PAC algorithm $\+A$ of $(m_1,\dots,m_K)$-\BAI with $\eps\in \tp{0,\frac{1}{8}}$, the expected number of pulls $T$ of $\+A$ on $\@H$ satisfies 
	\[
		\E{T} \ge c'\cdot\sum_{k=1}^T\frac{\log (m_k+1)}{\eps^2},
	\]
	where $c'>0$ is a universal constant. Moreover, we can pick $\@H$ as the one in which each arm follows $\!{Ber}(\frac{1}{2})$.
  	
\end{theorem}

Using the reduction from $\*m$-\BAI to $\*m$-\MAB (\Cref{lem: reduction}) again, we obtain the lower bound for $\*m$-\MAB.

\begin{theorem}\label{thm:lb-m-MAB}
	For any algorithm $\+A$ of $(m_1,\dots,m_k)$-\MAB, for any sufficiently large $T>0$, there exists a loss sequence $\ell^{(0)},\ell^{(1)},\dots,\ell^{(T-1)}$ such that the regret of $\+A$ in $T$ rounds is at least
	\[
		c'\cdot \sqrt{T\cdot\sum_{k=1}^K \log (m_k+1)},
	\]
	where $c'>0$ is a universal constant.
\end{theorem}

Our results generalize to the setting of \emph{bandit with graph feedback}. Let $G=(V,E)$ be a directed graph with self-loop on each vertex. Let $V_1,\dots,V_K \subseteq V$ be subsets of vertices. We say that they form a $(V_1,\dots,V_K)$-\emph{clique cover} of $G$ if each induced subgraph $G[V_k]$ for $k\in [K]$ is a clique and $\bigcup_{k\in [K]} V_k = V$.

\begin{corollary}\label{cor:ub-graph-bandit}
	Let $G$ be a feedback graph with a self-loop on each vertex. If $G$ contains a $(V_1,\dots,V_K)$-clique cover where $\abs{V_k}=m_k$ for every $k\in [K]$, then the minimax regret of bandit with graph feedback $G$ is at most
	\[ 
		c\cdot \sqrt{T\cdot\sum_{k=1}^K\log (m_k+1)}
	\]
\end{corollary}
for some universal constant $c>0$.


\bigskip
Our lower bounds generalize to bandit with graph feedback as well. The terms ``strongly observable feedback graphs'' and ``weakly observable feedback graphs'' are defined in \Cref{sec: prelim}.

\begin{theorem}\label{thm:lb-graph-bandit-strongly}
	Let $G=(V,E)$ be the feedback graph. Assume that there exist $K$ \emph{disjoint} sets $S_1,\dots ,S_K\subseteq V$ such that 
	\begin{itemize}
		\item each $G[S_k]$ is a strongly observable graph with a self-loop on each vertex;
		\item there is no edge between $S_i$ and $S_j$ for any $i\ne j$.
	\end{itemize}
	Then for any algorithm $\+A$ and any sufficiently large time horizon $T>0$, there exists some loss sequence on which the regret of $\+A$ is at least $c'\cdot \sqrt{T\cdot\sum_{k=1}^K\log {(\abs{S_k}+1)}}$ for some universal constant $c'>0$.
\end{theorem}


The following lower bound for weakly observable feedback graphs confirms a conjecture in~\cite{HZ22} and implies the optimality of several regret bounds established there, e.g., when the feedback graph is the disjoint union of loopless complete bipartite graphs. The notion of $t$-packing independent set is defined in \Cref{sec: prelim}.

\begin{theorem}\label{thm:lb-graph-bandit-weakly}
	Let $G=(V,E)$ be the feedback graph. Assume that $V$ can be partitioned into $K$ disjoint sets $V=V_1 \cup V_2 \cup \dots\cup V_K$ such that 
	\begin{itemize}
		\item for every $k\in [K]$, each $G[V_k]$ is observable;
		\item for every $k\in [K]$, there exists a $t_k$-packing independent set $S_k$ in $G[V_k]$ such that every vertex in $S_k$ does not have a self-loop;
		\item there is no edge from $V_i$ to $S_j$ for any $i\ne j$ in $G$.
	\end{itemize}
	Then for any algorithm $\+A$ and any sufficiently large time horizon $T>0$, there exists some loss sequence on which the regret of $\+A$ with feedback graph $G$ is at least $c'\cdot T^{\frac{2}{3}}\cdot \tp{\sum_{k=1}^K\max\set{\log {\abs{S_k}}, \frac{\abs{S_k}}{t_k}}}^{\frac{1}{3}}$ for some universal constant $c'>0$.
\end{theorem}


\Cref{thm:lb-graph-bandit-weakly} implies tight regret lower bounds for several weakly observable graphs. We summarize the minimax regret for some feedback graphs, weakly or strongly observable, in \Cref{tab: result-comp}.
\begin{table*}[htbp]
	\centering
	\caption{Minimax Regret Bound on Various Feedback Graphs}
	\label{tab: result-comp}
  \begin{threeparttable}
\begin{tabular}{m{4cm}<{\centering}m{5cm}<{\centering}m{6cm}<{\centering}}
	\toprule
	{Graph Type} & {Previous Result} & {This Work}\\
	\midrule
	General strongly observable graphs with self-loops & \shortstack{$O\tp{\sqrt{\alpha T}\log NT}$\\$\Omega\tp{\sqrt{\alpha T}}$}\tnote{1} & \shortstack{$O\tp{\sqrt{T\sum_{k=1}^K \log m_k}}$ \tnote{2}\\ See \Cref{thm:lb-graph-bandit-strongly} for the lower bound} \\
	Disjoint union of $K$ cliques  & \shortstack{$O\tp{\sqrt{K T}\log NT}$ \\ $\Omega\tp{\sqrt{K T}}$} & \shortstack{$\Theta\tp{\sqrt{T\sum_{k=1}^K \log m_k}}$} \\
	General weakly observable graphs & $\Omega\tp{T^{\frac{2}{3}}\max\set{\frac{|S|}{k},\log \abs{S}}^{\frac{1}{3}}}$ \tnote{3} & $\Omega\tp{T^{\frac{2}{3}}\tp{\sum_{k=1}^K\max\set{\log {\abs{S_k}}, \frac{\abs{S_k}}{t_k}}}^{\frac{1}{3}}}$ \\
	Disjoint union of $K$ loopless bipartite graphs & $\Omega\tp{T^{\frac{2}{3}}\tp{\log N}^{\frac{1}{3}}}$ & $\Omega\tp{T^{\frac{2}{3}} \tp{\sum_{k=1}^K \log m_k}^{\frac{1}{3}}}$ \\
	\bottomrule
\end{tabular}
\begin{tablenotes}
	\footnotesize
	\item[1] Here $\alpha$ is the independence number of the graph.
	\item[2] Here $K$ is the clique cover number of the graph and $m_1,m_2,\dots m_K$ are the size of the $K$ cliques respectively.
	\item[3]  Here $S$ is a $t$-packing independent set of the graph. $S_k$ and $t_k$ are defined in  \Cref{thm:lb-graph-bandit-weakly}. 
	\item[4] Previous results are from~\cite{ACDK15},~\cite{ACGMMS17} and~\cite{CHLZ21}.
  \end{tablenotes}
\end{threeparttable}
\end{table*}


\subsection{Overview of Technique} \label{sec: tech}

We note that a simple reduction (\Cref{lem: reduction}) implies that any algorithm for $\*m$-\MAB can be turned into a PAC algorithm for $\*m$-\BAI. As a result, \Cref{thm:ub-m-MAB,thm:ub-m-BAI,thm:lb-m-BAI,thm:lb-m-MAB} follow from a minimax regret upper bound for $\*m$-\MAB and a lower bound for $\*m$-\BAI.

\subsubsection{Upper bounds for \texorpdfstring{$\*m$}~-\MAB}

We design a new two-stage algorithm (Algorithm~\ref{algo:main-algo}) to establish an upper bound for $\*m$-\MAB. The algorithm is similar to the one used in~\cite{HZ22} to study weakly observable graphs with a few tweaks to incorporate our new analysis.


The algorithm maintains a distribution over $K$ groups and for each group, it maintains a distribution for arms in that group. In each round of the game, the algorithm pulls an arm in a two-stage manner: First pick the group according to the distribution over groups and then pick the arm in that group following the distribution in the group. At the end of each round, all distributions are updated in the manner similar to \emph{online stochastic mirror descent} (OSMD) with carefully designed loss vectors and various potential functions.

Our main technical contribution is a novel analysis of this two-stage algorithm. We design auxiliary two-stage \emph{piecewise continuous processes} whose regret is relatively easy to analyze. Then we view our algorithm as a discretization of the process and bound the accumulated discretization errors. 

Since the notion of $\*m$-\MAB generalizes both learning with expert advice and multi-armed bandit, we remark that our analysis of Algorithm~\ref{algo:main-algo} can specialize to an analysis of both ordinary mirror descent (MD) algorithm and OSMD algorithm. We believe that the viewpoint of discretizing a piecewise continuous process is more intuitive than the textbook analysis of OSMD and may be of independent pedagogical interest. 


\subsubsection{Lower bounds for \texorpdfstring{$\*m$}~-\BAI}

Our lower bound for the number of rounds in an $(\eps,0.05)$-PAC algorithm for $\*m$-\BAI where $\*m=(m_1,\dots,m_K)$ is 
\[
	\Omega\Big(\sum_{k=1}^K \frac{\log (m_k+1)}{\eps^2}\Big),
\]
which is the sum of lower bounds on each $(m_k)$-\BAI instance. To achieve this, we show that the instance where all arms are $\!{Ber}(\frac{1}{2})$ is in fact a universal hard instance in the sense that every $(\eps,0.05)$-PAC algorithm requires $\Omega\Big(\sum_{k=1}^K \frac{\log (m_k+1)}{\eps^2}\Big)$ to identify. Via a reduction of ``direct-sum'' flavor, we show that every $(\eps,0.05)$-PAC algorithm, when applied to this instance, must successfully identify that each group consists of $\!{Ber}(\frac{1}{2})$ arms. As a result, the lower bound is the sum of the lower bounds for each ``all $\!{Ber}(\frac{1}{2})$'' $(m_k)$-\BAI instance. 

We then prove the lower bound for ``all $\!{Ber}(\frac{1}{2})$'' $(m)$-\BAI instance for every $m\ge 2$. We use $\@H_0^{(m)}$ to denote this instance. The $\@H_0^{(m)}$ specified lower bound is obtained by constructing another $m$ instances $\@H_1^{(m)},\dots,\@H_m^{(m)}$ and compare the distribution of losses generated by $\@H_0^{(m)}$ and the distribution of losses generated by a \emph{mixture} of $\@H_1^{(m)},\dots,\@H_m^{(m)}$. For technical reasons, we first prove the lower bound when all arms are Gaussian and reduce the Gaussian  arms to Bernoulli arms.

\subsection{Organization of the Paper}
In this paper, we focus on the $\*m$-MAB and the $\*m$-\BAI and provide a fine-grained analysis to achieve  tight bounds for both problems. The paper is organized in the following way. We outline our main results in \Cref{sec: main results} and introduce the preliminaries in \Cref{sec: prelim}. A two-stage optimal algorithm for $\*m$-\MAB is given in \Cref{sec: upper bound}, along with continuous-time and discretized analysis. We then generalize this result to bandit with strongly observable graphs in \Cref{subsec:strongly-ub}. We also construct an $(\eps,0.05)$-\PAC algorithm for $\*m$-\BAI which terminates in bounded rounds in \Cref{subsec:bai_ub} via a reduction to $\*m$-\MAB problems. 

In \Cref{sec: BAI lb}, we derive a corresponding lower bound for $\*m$-\BAI. Based on the results in \Cref{sec: BAI lb}, we provide a regret lower bound for $\*m$-\MAB in \Cref{subsec: regret-lb-MAB} which matches the upper bound in \Cref{sec: upper bound}. We also prove the lower bounds for bandit with strongly and weakly observable feedback graphs in \Cref{subsec: regret-lb-strongly} and \Cref{sec: weakly-regret-lb} respectively. The result on weakly observable graphs solves an open problem in~\cite{HZ22}.

\subsection{Related Works}
The bandit feedback setting as an online decision problem has received considerable attention. The work of~\cite{AB09} first provided a tight bound for the bandit feedback setting, while the full information feedback case has been well studied in~\cite{FS97, HKW95}. Building upon these works,~\cite{MS11} introduced an interpolation between these two extremes and generalized the feedback of the classic bandit problem to a graph structure. Several prior studies, such as~\cite{ACDK15, ZL19, CHLZ21, HZ22}, have proposed various graph parameters to characterize the factors that influence regret. However, the algorithms proposed in these works for more general graphs do not yield a tight bound in our specific setting.

The pure exploration version of the bandit problem, known as the \emph{best arm identification} (\BAI) problem, has also received significant attention in the literature (\cite{EMM02, MT04, BMS09,ABM10, KKS13, CLQ17}). While the \BAI problem may appear deceptively simple, determining the precise bound for \BAI under the bandit feedback setting remains an open question. However, for the problem of identifying an $\eps$-optimal arm with high probability,~\cite{EMM02} established a tight bound for the bandit feedback setting, while the bound for the full feedback model is relatively straightforward (see e.g.  \cite{CHLZ21}).

\subsubsection{Comparison with~\cite{EECCB23}}

The very recent work of~\cite{EECCB23} studied interpolation of learning with experts and multi-armed bandit as well from a different perspective. They proved an $O\tuple{\sqrt{T\alpha(1+\log \tuple{N/\alpha})}}$ upper bound for the minimax regret of bandit with strongly feedback graph $G$ where $\alpha$ is the \emph{independence number} of $G$. The parameter is in general \emph{not} comparable with clique covers used in this work for feedback graphs. Particularly on an $\*m$-\MAB instance where $\*m=(m_1,\dots,m_K)$, the independence number is $K$ and therefore their upper bound becomes to $O\tp{\sqrt{TK\log(N/K)}}$ while our results showed that the minimax regret is indeed $\Theta\tp{\sqrt{T\sum_{k=1}^K\log (m_k+1)}}$. To see the difference, assume $K=\lfloor\log N\rfloor$ and $\*m=(1,1,\dots,1,N-K+1)$, then the minimax regret is $\Theta\tp{\sqrt{T\log N}}$ while the upper bound in~\cite{EECCB23} is $O\tp{\sqrt{T}\log N}$.

\section{Preliminaries}\label{sec: prelim}
In this section, we formally define the notations used and introduce some preparatory knowledge that will help in understanding this work.

\subsection{Mathematical Notations} 
Let $n$ be a non-negative integer. We use $[n]$ to denote the set $\set{1,2,\dots,n}$ and $\Delta_{n-1}=\set{\*x\in\bb R_{\geq 0}^n:\sum_{i=1}^n \*x(i)=1}$ to denote the $n-1$ dimensional standard simplex where $\bb R_{\geq 0}$ is the set of all non-negative real numbers. For a real vector $\*x\in \bb R^n$, the $i$-th entry of $\*x$ is denoted as $\*x(i)$ for every $i\in [n]$. We define $\*e^{[n]}_i$ as the indicator vector of the $i$-th coordinate such that $\*e^{[n]}_i(i)=1$ and $\*e^{[n]}_i(j)=0$ for all $j\neq i$ and $j\in[n]$. We may write $\*e^{[n]}_i$ as $\*e_i$ if the information on $n$ is clear from the context.

Given two vectors $\*x,\*y\in \bb R^n$, we define their inner product as $\inner{\*x}{\*y}=\sum_{i=1}^n \*x(i)\*y(i)$. For any $a,b\in \bb R$, let $[a,b]=\set{c\in \bb R\mid \min\set{a,b}\le c\le \max\set{a,b}}$ be the interval between $a$ and $b$. For any $\*x,\*y\in \bb R^n$, we say $\*y\geq \*x$ if $\*y(i)\geq \*x(i)$ for every $i\in[n]$. Then we can define the rectangle formed by $\*x$ and $\*y$: $\rect(\*x,\*y)=\set{\*z\in\bb R^n:\*y\geq \*z\geq \*x}$.

For any positive semi-definite  matrix $M\in \bb R^{n\times n}$, let $\|\*x\|_M=\sqrt{\*x^{\!T}M\*x}$ be the norm of $\*x$ with respect to $M$. Specifically, we abbreviate $\|\*x\|_{\tp{\nabla^2 \psi}^{-1}}$ as $\|\*x\|_{\nabla^{-2}\psi}$ where $\nabla^2 \psi$ is the Hessian matrix of a convex function $\psi$.

Let $F:\bb R^n\to \bb R$ be a convex function which is differentiable in its domain $\dom(F)$. Given $\*x,\*y \in\dom(F)$, the Bregman divergence  with respect to $F$ is defined as $B_F(\mathbf{x}, \mathbf{y})=F(\mathbf{x})-F(\mathbf{y})-\inner{\*x-\*y}{\nabla F(\*y)}$. 
Given two measures $\*{P}_1$ and $\*{P}_2$ on the same measurable space $(\Omega, \+F)$, the KL-divergence between $\*{P}_1$ and $\*{P}_2$ is defined as $\DKL\tp{\*{P}_1, \*{P}_2}=\sum_{\omega\in \Omega}\PP[1]{\omega}\log \frac{\PP[1]{\omega}}{\PP[2]{\omega}}$ if $\Omega$ is discrete or $\DKL\tp{\*{P}_1, \*{P}_2}=\int_\Omega\log \frac{\PP[1]{\omega}}{\PP[2]{\omega}}\dd{\PP[1]{\omega}}$ if $\Omega$ is continuous provided $\*P_1$ is absolutely continuous with respect to $\*P_2$.

\subsection{Graph Theory}
Let $G=(V,E)$ be a directed graph where $\abs{V}=N$. We use $(u,v)$ to denote the directed edge from vertex $u$ to vertex $v$. For any $U\subseteq V$, we denote the subgraph induced by $U$ as $G[U]$. For $v\in V$, let $\Nin(v)\defeq\set{u\in V\cmid (u,v)\in E}$ be the set of in-neighbors of $v$ and $\Nout(v)\defeq\set{u\in V\cmid (v,u)\in E}$ be the set of out-neighbors. If the graph is undirected, we have $\Nin(v)=\Nout(v)$, and we use $\-{N}(v)$ to denote the neighbors for brevity. We say $S\subseteq V$ is an independent set of $G$ if for every $v\in S$, $\set{u\in S\mid u\neq v,u\in \Nin(v)\cup\Nout(v)}=\emptyset$. The maximum independence number of $G$ is denoted as $\alpha(G)$ and abbreviated as $\alpha$ when $G$ is clear from the context. Furthermore, we say an independent set $S$ is a $t$-packing independent set if and only if for any $v\in V$, there are at most $t$ out-neighbors of $v$ in $S$, i.e., $\abs{\Nout(v)\cap S}\le t$. 
We say the subsets $V_1,\dots,V_K \subseteq V$ form a $(V_1,\dots,V_K)$-\emph{clique cover} of $G$ if each induced subgraph $G[V_k]$ for $k\in [K]$ is a clique and $\bigcup_{k\in [K]} V_k = V$.


\subsection{\texorpdfstring{$\*m$}~-\MAB and \texorpdfstring{$\*m$}~-\BAI}

Let $K>0$ be an integer. Given a vector $\*m=\tp{m_1,m_2,\dots,m_K}\in \bb Z_{\ge 1}^K$ with $\sum_{k\in [K]} m_k = N$, we now define problems $\*m$-\MAB and $\*m$-\BAI respectively.

\subsubsection{\texorpdfstring{$\*m$}~-\MAB}

In the problem of $\*m$-\MAB, there are $N$ arms. The arms are partitioned into $K$ groups and the $k$-th group contains $m_k$ arms. Let $T\in\bb N$ be the time horizon. Then $\*m$-\MAB is the following online decision game. The game proceeds in $T$ rounds. At round $t=0,1,\dots,T-1$:
\begin{itemize}
	\item The player pulls an arm $A_t\in [N]$;
	\item The adversary chooses a loss function $\ell^{(t)}\in [0,1]^N$;
	\item The player incurs loss $\ell^{(t)}(A_t)$ and observes the losses of all arms in the group containing $A_t$.
\end{itemize}
Clearly the vector $\*m$ encodes the amount of information the player can observe in each round. Two extremes are the problem of learning with expert advice and multi-armed bandit, which correspond to $(N)$-\MAB and $(1,\dots,1)$-\MAB respectively.

We assume the player knows $\*m$ and $T$ in advance and use $\+A$ to denote the player's algorithm (which can be viewed as a function from previous observed information and the value of its own random seeds to the arm pulled at each round).

The performance of the algorithm $\+A$ is measured by the notion of \emph{regret}. Fix a loss sequence $\vec{L}=\set{\ell^{(0)},\dots,\ell^{(T-1)}}$. Let $a^*= \argmin_{a\in [N]}\sum_{t=1}^T\ell^{(t)}(a)$ be the arm with minimum accumulated losses. The regret of the algorithm $\+A$ and time horizon $T$ on $\vec{L}$ with respect to the arm $a$ is defined as $R_a(T,\+A,\vec{L})=\E{\sum_{t=0}^{T-1}\ell^{(t)}(A_t)}-\sum_{t=0}^{T-1} \ell^{(t)}(a)$. If there is no ambiguity, we abbreviate $R_a(T,\+A,\vec{L})$ as $R_a(T)$. We also use $R(T)$ to denote $R_{a^*}(T)$.

We are interested in the regret of the best algorithm against the worst adversary, namely the quantity
\[
	R^*_a(T)=\inf_{\+A}\sup_{\vec{L}}R_a(T,\+A,\vec{L}).
\]
We call $R^*_{a^*}(T)$ the \emph{minimax regret} of $\*m$-\MAB and usually write it as $R^*(T)$. 


\bigskip
We may use the following two ways to name an arm in $\*m$-\MAB:
\begin{itemize}
	\item use the pair $(k,j)$ where $k\in [K]$ and $j\in [m_k]$ to denote ``the $j$-th arm in the $k$-th group'';
	\item use a global index $i\in [N]$ to denote the $i$-th arm.
\end{itemize}
Following this convention, we use $\ell^{(t)}(i)$ and $\ell^{(t)}_k(j)$ to denote the loss of arm $i$ and arm $(k,j)$ at round $t$ respectively. 

\subsubsection{Best Arm Identification and \texorpdfstring{$\*m$}~-\BAI}
The \emph{best arm identification} (\BAI) problem asks the player to identify the best arm among $N$ given arms with as few pulls as possible. To be specific, each arm $i$ is associated with a parameter $p_i$ and each pull of arm $i$ gives an observation of its random loss, which is drawn from a fixed distribution with mean $p_i$ independently. The loss of each arm is restricted to be in $[0,1]$. 
The one with smallest $p_i$, indexed by $i^*$, is regarded as the best arm. An arm $j$ is called an \emph{$\eps$-optimal arm} if its mean is less than the mean of the best arm plus $\eps$ for some $\eps\in (0,1)$, namely $p_j< p_{i^*}+\eps$. 
With fixed $\eps,\delta>0$, an $(\eps, \delta)$-\emph{probably approximately correct} algorithm, or $(\eps,\delta)$-PAC algorithm for short, can find an $\eps$-optimal arm with probability at least $1-\delta$. In most parts of this paper, we choose $\delta=0.05$. For an algorithm $\+A$ of \BAI, we usually use $T$ to denote the number of arms $\+A$ pulled before termination. Similarly for any arm $i$, we use $T_{i}$ to denote the number of times that the arm $i$ has been pulled by $\+A$ before its termination. We also use $N_{i}$ to denote the number of times that the arm $i$ has been \emph{observed} by $\+A$. 


Let $\*m=\tp{m_1,m_2,\cdots,m_K}\in \bb Z_{\ge 1}^K$ be a vector. Similar to $\*m$-\MAB, the arms are partitioned into $K$ groups and the $k$-th group consists of $m_k$ arms. Each pull of an arm can observe the losses of all arms in the group. As usual, the goal is to identify the best arm (the one with minimum $p_i$) with as few rounds as possible.


Similar to $\*m$-\MAB, we use $i\in [N]$ or $(k,j)$ where $k\in [K]$ and $j\in [m_k]$ to name an arm. For a fixed algorithm, we use $T_i$ or $T_{(k,j)}$ to denote the number of times the respective arm has been pulled and use $N_i$ or $N_{(k,j)}$ to denote the number of times it has been observed. For every $k\in [K]$ we use $T^{(k)}$ to denote the number of times the arms in the $k$-th group have been pulled, namely $T^{(k)} = \sum_{j\in [m_k]} T_{(k,j)}$. By definition, it holds that $T=\sum_{k\in [K]} T^{(k)}$ and $N_{(k,j)}=T^{(k)}$ for every $j\in [m_k]$.


\subsection{Bandit with Graph Feedback}\label{sec:def}

A more general way to encode the observability of arms is to use feedback graphs. In this problem, a directed graph $G=(V,E)$ is given. The vertex set $V=[N]$ is the collection of all arms. 

The game proceeds in the way similar to $\*m$-\MAB. The only difference is that when an arm $A_t$ is pulled by the player at a certain round, all arms in $\Nout(A_t)$ can be observed. As a result, given a vector $\*m=\tp{m_1,m_2,\cdots,m_K}\in \bb Z_{\ge 1}^K$, the $\*m$-\MAB problem is identical to bandit with graph feedback $G=(V,E)$ where $G$ is the disjoint union of $K$ cliques $G_1=(V_1,E_1),G_2=(V_2,E_2),\dots,G_K=(V_K,E_K)$ with $m_k=\abs{V_k}$ and $E_k=V_k^2$ for every $k\in [K]$.

According to~\cite{ACDK15}, we measure the observability of each vertex in terms of its in-neighbors. If a vertex has no in-neighbor, we call it a \emph{non-observable} vertex, otherwise it is \emph{observable}. If a vertex $v$ has a self-loop \emph{or} $\Nin(v)$ exactly equals to $V\setminus\set{v}$, then $v$ is \emph{strongly observable}. If an observable vertex is not strongly observable, then it is \emph{weakly observable}. In this work, we assume each vertex is observable. If all the vertices are strongly observable, the graph $G$ is called a strongly observable graph. If $G$ contains weakly observable vertices (and does not have non-observable ones), we say $G$ is a weakly observable graph.

We can also define the notion of regret for bandit with graph feedback. Assume notations before, the regret of an algorithm $\+A$ with feedback graph $G$ and time horizon $T$ on a loss sequence $\vec{L}$ with respect to the arm $a$ is defined as $R_a(G,T,\+A,\vec{L})=\E{\sum_{t=0}^{T-1}\ell^{(t)}(A_t)}-\sum_{t=0}^{T-1} \ell^{(t)}(a)$. If there is no ambiguity, we abbreviate $R_a(G,T,\+A,\vec{L})$ as $R_a(G,T)$ or $R_a(T)$. We also use $R(T)$ to denote $R_{a^*}(T)$. Then minimax regret is again

\[
	R^*_{a^*}(G,T)=\inf_{\+A}\sup_{\vec{L}}R_{a^*}(G, T,\+A,\vec{L}).
\]
When $G$ is clear from the context, we write it as $R^*(T)$.


\section{The Upper Bounds}\label{sec: upper bound}

In this section, we prove \Cref{thm:ub-m-MAB} and \Cref{thm:ub-m-BAI}. We describe the algorithm for $\*m$-\MAB in \Cref{sec:algo} and analyze it in \Cref{sec:analysis}. The algorithm for $\*m$-\BAI is obtained by a reduction to $\*m$-\MAB described in \Cref{subsec:bai_ub}. Finally we discuss how to extend the algorithm to bandit with strongly observable feedback graphs and prove \Cref{cor:ub-graph-bandit} in \Cref{subsec:strongly-ub}. 

\subsection{The Algorithm} \label{sec:algo}
As discussed in the introduction, our algorithm basically follows the framework of the two-stage online stochastic mirror descent developed in~\cite{HZ22}. However, our updating rules is slightly different from the one in~\cite{HZ22} in order to incorporate with our new analysis. 

Given a $K$-dimensional vector $\*m=(m_1,\dots,m_K)$ as input, in each round $t$, the algorithm proceeds in the following two-stage manner:
\begin{itemize}
	\item A distribution $Y^{(t)}$ over $[K]$ is maintained, indicating which group of arms the algorithm is going to pick.
	\item For each $k\in [K]$, a distribution $X^{(t)}_k$ is maintained, indicating which arm in the $k$-th group the algorithm will pick conditioned on that the $k$-th group is picked in the first stage. 
	\item The algorithm then picks the $j$-th arm in the $k$-group with probability $Y^{(t)}(k)\cdot X^{(t)}_k(j)$.
\end{itemize}

The algorithm is described in Algorithm~\ref{algo:main-algo} and we give an explanation for each step below. Assuming $Y^{(0)}$ and $X_k^{(0)}$ for all $k\in [K]$ are well initialized, in each time step $t= 0,1,\dots ,T-1$,  the player  will repeat the following operations:


\paragraph{Sampling:} For each arm $(k,j)$, the algorithm pulls it with probability
\[
{Z}^{(t)}(k,j)={Y}^{(t)}(k)\cdot X_k^{(t)}(j).
\]
The arm pulled at this round is denoted by $A_t=(k_t,j_t)$. Our algorithm can guarantee that $Z^{(t)}$ is a distribution over all arms.

\paragraph{Observing:} Observe partial losses $\ell^{(t)}_{k_t}(j)$ for all $j\in [m_{k_t}]$.

\paragraph{Estimating:} For each arm $(k,j)$, define the unbiased estimator $\hat{\ell}_k^{(t)}(j) =  \frac{\1{k=k_t}}{\Pr{k=k_t}}\cdot \ell_k^{(t)}(j)$. It is clear that $\E{\hat{\ell}_k^{(t)}(j)} = \ell_k^{(t)}(j)$.

\paragraph{Updating:}
\begin{itemize}
	\item For each $k\in[K]$, update $X^{(t)}_k$ in the manner of standard OSMD:
	\[
	\nabla \phi_{k}(\ol{X}_k^{(t+1)})=\nabla\phi_k(X_k^{(t)})-\hat{\ell}_k^{(t)}; \quad X_k^{(t+1)}= \argmin_{\*x\in\Delta_{m_k-1}} B_{\phi_k}(\*x,\ol{X}_k^{(t+1)}),
	\]
	%
	%
	%
	where $\phi_k(\*x)=\eta_k^{-1}\sum_i x(i)\log x(i)$ is the negative entropy scaled by the learning rate $\eta_k$.
	\item Define $\ol{Y}^{(t)}$ in the way that 
	\begin{equation}\label{eqn:Yupdate}
		\frac{1}{\sqrt{\ol{Y}^{(t+1)}(k)}}=\frac{1}{\sqrt{Y^{(t)}(k)}}+\sum_{j\in[m_k]}\frac{\eta}{\eta_k}X_k^{(t)}(j)\tuple{1-\exp\tp{-\eta_k\cdot \hat{\ell}_k^{(t)}(j)}},\quad\forall k\in [K]			
	\end{equation}
	where $\eta$ is the learning rate. Then let $Y^{(t+1)}$ be the projection of $\ol{Y}^{(t+1)}$ on $\Delta_{K-1}$:
	\[
	Y^{(t+1)}=\argmin_{\*y\in\Delta_{K-1}} B_{\psi}(\*y,\ol{Y}^{(t+1)}),
	\]
	where $\psi(\*y)=-2\sum_i \sqrt{y(i)}$ for any $\*y=(y(1),\dots,y(K))\in\bb R^K$, referred to as Tsallis entropy in literature. Note that when $x$ is small, $1-\exp\tp{-x}\approx x$. So when $\eta_k$ is small (and it is so), the updating rule is approximately
	\[
	\frac{1}{\sqrt{\ol{Y}^{(t+1)}(k)}}=\frac{1}{\sqrt{Y^{(t)}(k)}}+\eta \sum_{j\in[m_k]}X_k^{(t)}(j)\cdot\hat\ell_k^{(t)}(j),\quad \forall k\in [K],
	\]
	which is equivalent to 
	\[
	\grad\psi(\ol{Y}^{(t+1)}) = \grad\psi(\ol{Y}^{(t)}) - \eta \cdot \wh L^{(t)},
	\]
	where $\wh L^{(t)}=(\wh L^{(t)}(1),\dots,\wh L^{(t)}(K))\in\bb R^K$ satisfying $\wh L^{(t)}(k)=\sum_{j\in [m_k]} X^{(t)}_k(j)\cdot\hat\ell_k^{(t)}(j)$. One can think of $\wh{L}^{(t)}(k)$ as the ``average loss'' of the arms in the $k$-th group at round $t$. Nevertheless, we use rule ~\eqref{eqn:Yupdate} in the algorithm since it is convenient for our analysis later.
\end{itemize}

In the realization of Algorithm~\ref{algo:main-algo}, we will choose $\eta=\frac{1}{\sqrt{T}}$ and  $\eta_k=\frac{\log \tp{m_k+1}}{\sqrt{T\sum_{k=1}^K\log (m_k+1)}}$.


\begin{algorithm}[ht]
	\KwIn{An $(m_1,\dots,m_K)$-\MAB instance}
	\BlankLine
	$X_{{k}}^{(0)} \leftarrow \mathop{\arg\min}\limits_{a\in \Delta_{m_k-1}} \phi_k(a)$, for all ${k}\in [K]$\;
	${Y}^{(0)}\leftarrow \mathop{\arg\min}\limits_{b\in \Delta_{K-1}} \psi(b)$\;
	\BlankLine
	\For{$t\leftarrow 0$ \KwTo $T-1$}{
		\BlankLine
		$Z^{(t)}(k,j)\leftarrow Y^{(t)}(k)\cdot{X}_k^{(t)}(j)$, for all $k\in [K]$ and $j\in [m_{{k}}]$\;
		\BlankLine
		Pull $A_t=(k_t,j_t)\sim Z^{(t)}$ and observe $\ell_{k_t}^{(t)}(j)$ for all $j\in [m_k]$\;
		\BlankLine
		$\forall {k}\in [K],\forall j\in [m_{{k}}]: \hat{\ell}^{(t)}_{{k}}(j)\leftarrow \frac{\1{k=k_t}}{\sum_{j\in [m_{k}]}Z^{(t)}((k,j))} \ell_k^{(t)}(j)=\frac{\1{k=k_t}}{Y^{(t)}({k})}\ell_k^{(t)}(j)$\;
		\BlankLine
		Update $\nabla \phi_{k}(\ol{X}_k^{(t+1)})=\nabla\phi_k(X_k^{(t)})-\hat{\ell}_k^{(t)}$\;
		\BlankLine
		$X_k^{(t+1)}= \argmin_{\*x\in\Delta_{m_k-1}} B_{\phi_k}(\*x,\ol{X}_k^{(t+1)})$\;
		\BlankLine		
		Update $\frac{1}{\sqrt{\ol{Y}^{(t+1)}(k)}}=\frac{1}{\sqrt{Y^{(t)}(k)}}+\frac{\eta}{\eta_k} \sum_{j\in[m_k]}X_k^{(t)}(j)\tuple{1-\exp\tp{-\eta_k\cdot \hat{\ell}_k^{(t)}(j)}}, \forall k\in [K]$\; \label{algo:line:Y}
		\BlankLine
		$Y^{(t+1)}=\argmin_{\*y\in\Delta_{K-1}} B_{\psi}(\*y,\ol{Y}^{(t+1)})$\;
	}
	\caption{Two-Stage Algorithm for $\*m$-\MAB}
	\label{algo:main-algo} 
\end{algorithm}

\subsection{Analysis} \label{sec:analysis}

We prove the following theorem, which implies \Cref{thm:ub-m-MAB}.

\begin{theorem}\label{thm:ub_clique}
	For every $T>0$ and every loss sequence $\ell^{(0)},\dots,\ell^{(T-1)}\in [0,1]^N$, the regret of Algorithm~\ref{algo:main-algo} satisfies
	\[
	R(T)\leq O\tuple{\sqrt{T\sum_{k=1}^K\log (m_k+1)}}.
	\]
\end{theorem}

Instead of directly bounding the regret of the sequence of the action distributions $\set{Z^{(t)}}_{0\le t\le T-1}$, we study an auxiliary \emph{piecewise continuous} process $\set{\+Z^{(s)}}_{s\in [0,T)}$. We define and bound the \emph{regret} of $\set{\+Z^{(s)}}_{s\in [0,T)}$ in \Cref{sec:continuous}, and compare it with the regret of $\set{Z^{(t)}}_{0\le t\le T-1}$ in \Cref{sec:compare}. Finally, we prove \Cref{thm:ub_clique} in \Cref{sec:proof}

\subsubsection{The piecewise continuous process}\label{sec:continuous}

Assuming notations in Algorithm~\ref{algo:main-algo}, the process $\set{\+Z^{(s)}}_{s\in [0,T)}$ is defined as
\[
\+Z^{(s)}(k,j) = \+Y^{(s)}(k)\cdot\+X^{(s)}_k(j), \quad\forall k\in [K], j\in [m_k],
\]
where $\set{\+Y^{(s)}}_{s\in [0,T)}$ and $\set{\+X^{(s)}_k}_{s\in [0,T)}$ for every $k\in [K]$ are piecewise continuous processes defined in the following way.

\begin{itemize}
	\item For every integer $t\in\set{0,1,\dots,T-1}$, we let $\+Y^{(t)} = Y^{(t)}$ and $\+X^{(t)}_k = X^{(t)}_k$ for every $k\in [K]$.
	\item For every integer $t\in\set{0,1,\dots,T-1}$ and every $k\in [K]$, the trajectory of $\set{\+X^{(s)}_k}_{s\in [t,t+1)}$ is a continuous path in $\bb R^{m_k}$ governed by the ordinary differential equation
	\begin{equation}\label{eqn:ode-X}
		\dv{\grad\phi_k(\+X^{(s)}_k)}{s} = -\hat\ell^{(t)}_k.
	\end{equation}
	
	\item For every integer $t\in\set{0,1,\dots,T-1}$, the trajectory of $\set{\+Y^{(s)}}_{s\in[t,t+1)}$ is a continuous path in $\bb R^K$ governed by the ordinary differential equation
	\begin{equation}\label{eqn:ode-Y}
		\dv{\grad\psi(\+Y^{(s)})}{s} = -\wh L^{(s)},
	\end{equation}
	where $\wh L^{(s)}=\tp{\wh L^{(s)}(1),\dots,\wh L^{(s)}(K)}\in \bb R^K$ satisfies $\wh L^{(s)}(k) = \sum_{j\in [m_k]} \+X_k^{(s)}(j)\cdot\hat\ell_k^{(t)}(j)$.
	
\end{itemize}

Clearly the trajectories of $\+Z^{(s)}$, $\+Y^{(s)}$ and $\+X^{(s)}_k$ for every $k\in [K]$ are piecewise continuous paths in the time interval $s\in [0,T)$. An important property is that the end of each piece of the trajectories of $\+Y^{(s)}$ and $\+X^{(s)}_k$ coincides with its discrete counterpart \emph{before} performing projection to the probability simplex. 

Formally, for every $t\in [T]$ and $k\in [K]$, define $\+X^{(t)^-}_k\defeq \lim_{s\to t^-}\+X^{(s)}_k$ and $\+Y^{(t)^-}\defeq \lim_{s\to t^-}\+Y^{(s)}$. We have the following lemma.

\begin{lemma}\label{lem:meet}
	For every $t\in [T]$ and $k\in [K]$, it holds that $\+X^{(t)^-}_k = \ol{X}^{(t)}_k$ and $\+Y^{(t)^-} = \ol{Y}^{(t)}$.
\end{lemma}
\begin{proof}
	To ease the notation, for any fixed $t\in \set{0,1,\dots,T-1}$ and fixed $k\in[K]$, we now prove that $\+X^{(t+1)^-}_k = \ol{X}^{(t+1)}_k$ and $\+Y^{(t+1)^-} = \ol{Y}^{(t+1)}$ respectively.
	
	In fact, $\+X^{(t+1)^-}_k = \ol{X}^{(t+1)}_k$ immediately follows by integrating both sides of~\eqref{eqn:ode-X} from $t$ to $t+1$ and noting that $\+X^{(t)}_k=X^{(t)}_k$.
	
	More efforts are needed to prove the identity for $\+Y^{(t)}$. Recall $\phi_k(\*x) = \eta_k^{-1}\sum_j x(j)\log x(j)$ for every $\*x=\big(x(1),\dots,x(m_k)\big)$. It follows from~\eqref{eqn:ode-X} that for every $s\in [t,t+1)$ every $k\in [K]$ and every $j\in [m_k]$, 
	\[
	\+X^{(s)}_k(j) = \+X^{(t)}_k(j)\cdot \exp\tp{-(s-t)\eta_k\hat\ell_k^{(t)}(j)}.
	\]
	As a result, we know that
	\[
	\wh L^{(s)}(k) = \sum_{j\in [m_k]} \+X_k^{(t)}(j)\cdot \exp\tp{-(s-t)\eta_k\hat\ell_k^{(t)}(j)}\cdot \hat\ell^{(t)}_k(j).
	\]
	Integrating~\eqref{eqn:ode-Y} from $t$ to $s$, plugging in above and noting that $\+Y^{(t)}=Y^{(t)}$, we obtain
	\[
	\frac{1}{\sqrt{\+Y^{(s)}(k)}} = \frac{1}{\sqrt{Y^{(t)}(k)}} + \frac{\eta}{\eta_k} \sum_{j\in[m_k]}X_k^{(t)}(j)\tuple{1-\exp\tp{-\eta_k\cdot (s-t)\cdot \hat{\ell}_k^{(t)}(j)}},
	\]
	which is exactly our rule to define $\ol{Y}^{(t+1)}$ in Line~\ref{algo:line:Y} of Algorithm~\ref{algo:main-algo} (take $s=t+1$).
\end{proof}

We define the regret for the piecewise continuous process as follows.
\begin{definition}
	The \emph{continuous regret} contributed by the process $\set{\+Z^{(s)}}_{s\in [0,T)}$ with respect to a fixed arm $a\in [N]$ is defined as
	\[
	\@R_a(T) \defeq \sum_{t=0}^{T-1}\E{\int_{t}^{t+1} \inner{\+Z^{(s)}-\*e_a^{[N]}}{\ell^{(t)}}\dd{s}}.
	\]
\end{definition}

Then we are ready to bound $\@R_a(T)$. Recall that we may write $\*e_a^{[N]}$ as $\*e_a$ if the information on $N$ is clear from the context.

\begin{lemma}\label{lem:continuous-regret}
	For any time horizon $T>0$, any loss sequence $\ell^{(0)},\ell^{(1)},\dots,\ell^{(T-1)}\in [0,1]^N$, and any arm $a=(k,j)$, it holds that
	\[
	\@R_a(T) \le B_\psi(\*e_k^{[K]}, Y^{(0)}) + B_{\phi_k}(\*e_j^{[m_k]},X^{(0)}_k).
	\]
\end{lemma}
\begin{proof}
	Assume $a=(k,j)$. For every $t\in \set{0,1,\dots,T-1}$, we compute the decreasing rate of the Bregman divergence caused by the evolution of $\+Y^{(s)}$ and $\+X^{(s)}_k$ respectively.
	
	First consider the change of $B_{\psi}(\*e_k,\+Y^{(s)})$ over time:
	\begin{align*}
		\dv{}{s}B_{\psi}(\*e_k,\+Y^{(s)})
		&=\dv{}{s}\tp{\psi(\*e_k)-\psi(\+Y^{(s)})-\inner{\*e_k-\+Y^{(s)}}{\grad\psi(\+Y^{(s)})}}\\
		&=\inner{\dv{\grad\psi(\+Y^{(s)})}{s}}{\+Y^{(s)}-\*e_k}\\
		&=-\inner{\wh L^{(s)}}{\+Y^{(s)}-\*e_k}.
	\end{align*}
	Integrating above from $t$ to $t+1$, we have
	\begin{equation}\label{eqn:Ycont}
		\int_{t}^{t+1}\inner{\wh L^{(s)}}{\+Y^{(s)}-\*e_k}\dd{s}=B_\psi(\*e_k, \+Y^{(t)}) - B_\psi(\*e_k,\+Y^{(t+1)^-}) = B_\psi(\*e_k, Y^{(t)}) - B_\psi(\*e_k,\ol{Y}^{(t+1)}),
	\end{equation}
	where the last equality follows from \Cref{lem:meet}. 
	
	Note that \emph{projection never increases Bregman divergence}; that is, we have
	\begin{align*}
		&\phantom{{}={}}B_\psi(\*e_{k},\ol{Y}^{(t+1)})- B_\psi(\*e_{k},Y^{(t+1)})\\
		&=\psi(Y^{(t+1)})-\psi(\ol{Y}^{(t+1)})+\inner{\grad\psi(Y^{(t+1)})}{\*e_k-Y^{(t+1)}}-\inner{\grad\psi(\ol Y^{(t+1)})}{\*e_k-\ol Y^{(t+1)}}\\	
		&=\underbrace{\psi(Y^{(t+1)}) - \psi(\ol Y^{(t+1)})-\inner{\grad \psi(\ol Y^{(t+1)})}{Y^{(t+1)}-\ol Y^{(t+1)}}}_A
		+\underbrace{\inner{\grad\psi(\ol Y^{(t+1)})-\grad\psi(Y^{(t+1)})}{Y^{(t+1)}-\*e_k}}_B.
	\end{align*}
	Since $\psi$ is convex, we have $A\ge 0$. By the definition of $Y^{(t+1)}$,
	\[
	Y^{(t+1)}=\argmin_{\*y\in \Delta_{K-1}} B_{\psi}(\*y,\ol Y^{(t+1)}) = \argmin_{\*y\in\Delta_{K-1}} \psi(\*y) - \inner{\*y}{\grad\psi(\ol Y^{(t+1)})}.
	\]
	The first-order optimality condition (see Section 26.5 in~\cite{LS20}) implies that $B\ge 0$. 
	As a result, $B_\psi(\*e_{k},\ol{Y}^{(t+1)}) \ge B_\psi(\*e_{k},Y^{(t+1)})$ and it follows from \Cref{eqn:Ycont} that 
	\begin{equation}\label{eqn:Y-regret}
		\int_{t}^{t+1}\inner{\wh L^{(s)}}{\+Y^{(s)}-\*e_k}\dd{s} \le B_\psi(\*e_k, Y^{(t)}) - B_\psi(\*e_k,Y^{(t+1)}).
	\end{equation}
	
	Then we consider the change of $B_{\phi_k}(\*e_j,\+X^{(s)}_k)$ over time. Likewise we have
	\begin{align*}
		\dv{}{s}B_{\phi_k}(\*e_j,\+X^{(s)}_k)
		&=\inner{\dv{\grad\phi_k(\+X^{(s)}_k)}{s}}{\+X^{(s)}_k-\*e_j} = -\inner{\hat\ell^{(t)}_k}{\+X^{(s)}_k-\*e_j}.
	\end{align*}
	By an argument similar to the one for $\+Y^{(s)}$ above, we can obtain
	\begin{equation}\label{eqn:X-regret}
		\int_{t}^{t+1} \inner{\hat\ell^{(t)}_k}{\+X^{(s)}_k-\*e_j} \dd{s}\le B_{\phi_k}(\*e_j,X^{(t)}_k)-B_{\phi_k}(\*e_j,X^{(t+1)}_k).
	\end{equation}
	
	On the other hand, we have for every $s\in [t,t+1)$ and any arm $a^*=(k^*,j^*)$,
	\[
	\E{\inner{\+Z^{(s)}-\*e_{a^*}}{\ell^{(t)}}} 
	= \E{\inner{\+Z^{(s)}-\*e_{a^*}}{\hat\ell^{(t)}}}
	=\E{\sum_{k\in [K]}\sum_{j\in [m_k]}\+Y^{(s)}(k)\cdot \+X^{(s)}_k(j)\cdot\hat\ell^{(t)}_k(j)-\hat\ell^{(t)}(a^*)}.
	\]
	Recall that for every $k\in [K]$, it holds that $\wh L^{(s)}(k) = \sum_{j\in [m_k]} \+X^{(s)}_k(j)\cdot \hat\ell^{(t)}_k(j)$. Rearranging above yields
	\begin{align*}	
		\E{\inner{\+Z^{(s)}-\*e_{a^*}}{\ell^{(t)}}}
		&=\E{\sum_{k\in [K]} \+Y^{(s)}(k)\cdot\wh L^{(s)}(k) - \hat\ell^{(t)}(a^*)}\\
		&=\E{\inner{\+Y^{(s)}}{\wh L^{(s)}} - \hat\ell^{(t)}(a^*)}\\
		&=\E{\inner{\+Y^{(s)}-\*e_{k^*}}{\wh L^{(s)}}+\wh L^{(s)}(k^*) -\hat\ell^{(t)}_{k^*}(j^*)}\\
		&=\E{\inner{\+Y^{(s)}-\*e_{k^*}}{\wh L^{(s)}}} +\E{\inner{\+X^{(s)}_{k^*}-\*e_{j^*}}{\hat\ell^{(t)}_{k^*}}}.
	\end{align*}
	Integrating above from $t$ to $t+1$ and plugging in \Cref{eqn:Y-regret,eqn:X-regret}, we obtain
	\begin{align*}
		\int_{t}^{t+1}\E{\inner{\+Z^{(s)}-\*e_{a^*}}{\ell^{(t)}}}\dd{s} 
		&= \int_{t}^{t+1}\E{\inner{\+Y^{(s)}-\*e_{k^*}}{\wh L^{(s)}}}\;\dd s +\int_{t}^{t+1} \E{\inner{\+X^{(s)}_k-\*e_{j^*}}{\hat\ell^{(t)}_{k^*}}}\dd{s}\\
		&\le B_\psi(\*e_k, Y^{(t)}) - B_\psi(\*e_k,Y^{(t+1)}) + B_{\phi_k}(\*e_j,X^{(t)}_k)-B_{\phi_k}(\*e_j,X^{(t+1)}_k).
	\end{align*}
	Summing above over $t$ from $0$ to $T-1$ finishes the proof.
\end{proof}

\subsubsection{Comparison of \texorpdfstring{$R_a(T)$}~ and \texorpdfstring{$\@R_a(T)$}~}\label{sec:compare}

For any fixed loss sequence $\ell^{(0)}, \ell^{(1)},\dots,\ell^{(T-1)}$, we bound the difference between the regret $R_a(T)$ of Algorithm~\ref{algo:main-algo} and the continuous regret $\@R_a(T)$ for any arm $a$. Formally, we establish the following lemma:
\begin{lemma}\label{lem:compare}
	\[
	R_a(T)-\@R_a(T)
	\le \frac{1}{2}\sum_{t=0}^{T-1} \E{\sup_{\xi\in\rect(Y^{(t)},\ol Y^{(t+1)})} \|\wh L^{(t)}\|^2_{\grad^{-2}\psi(\xi)} + \sum_{k\in [K]} Y^{(t)}(k)\cdot \sup_{\zeta_k\in\rect(X^{(t)}_k,\ol X^{(t+1)}_k)} \|\hat\ell^{(t)}_k\|^2_{\grad^{-2}\phi_k(\zeta_k)}}.
	\]
\end{lemma} 

\begin{proof}
	By the definition of the regret, we have
	\begin{align*}
		R_a(T) 
		&= \E{\sum_{t=0}^{T-1}\inner{Z^{(t)}-\*e_a}{\hat\ell^{(t)}}} \\
		&= \sum_{t=0}^{T-1}\E{\inner{Z^{(t)}-\*e_a}{\hat\ell^{(t)}}} \\
		&= \sum_{t=0}^{T-1}\E{\int_{t}^{t+1}\inner{\+Z^{(s)}-\*e_a}{\hat\ell^{(t)}}\;\dd s + \int_{t}^{t+1}\inner{Z^{(t)}-\+Z^{(s)}}{\hat\ell^{(t)}}\;\dd s}\\
		&=\@R_a(T) + \sum_{t=0}^{T-1}\E{\int_{t}^{t+1}\inner{Z^{(t)}-\+Z^{(s)}}{\hat\ell^{(t)}}\;\dd s},
	\end{align*}
	where the first equality holds due to Fubini's theorem.
	Therefore, we only need to bound the term $\sum_{t=0}^{T-1} \E{\int_{t}^{t+1} \inner{Z^{(t)}-\+Z^{(s)}}{\hat\ell^{(t)}}\;\dd s}$.
	
	
	Fix $t\in\set{0,1,\dots,T-1}$. We have shown in the proof of \Cref{lem:meet} that 
	\[
	\+X_k^{(s)}(j)=X_k^{(t)}(j)\cdot \exp\tp{-(s-t)\eta_k\hat{\ell}_k^{(t)}(j)}\le X_k^{(t)}(j)
	\]
	for any $s\in [t,t+1)$ and any $j\in [m_k]$. 
	
	Recall that $\wh L^{(s)}(k) = \sum_{j\in [m_k]} \+X^{(s)}_k(j)\cdot \hat\ell^{(t)}_k (j)$ for every $k\in [K]$. Then by the discussion above, we have $\wh L^{(s)} \le \wh L^{(t)}$ for any $s\in [t,t+1)$. As a result, it follows from~\eqref{eqn:ode-Y} that for any $s\in [t,t+1)$,
	\begin{equation}\label{eqn:gradY-bound}
		\grad\psi (\+Y^{(s)})-\grad\psi(Y^{(t)})=\int_{t}^s -\wh{L}^{(w)}\;\dd w\ge -(s-t)\cdot \wh{L}^{(t)}.	
	\end{equation}
	
	Recall that for any two vectors $\*x,\*y$ of the same dimension, $\rect(\*x,\*y)$ is the rectangle between $\*x$ and $\*y$. Since our $\psi$ is a \emph{separable function} (and therefore $\grad^2\psi$ is diagonal), we can apply the \emph{mean value theorem} entrywise and obtain 
	\begin{equation}\label{eqn:gradY-Hessian-bound}
		\grad\psi (\+Y^{(s)})-\grad\psi(Y^{(t)})=\grad^2\psi(\xi^{(s)})(\+Y^{(s)}-Y^{(t)})	
	\end{equation}
	for some $\xi^{(s)}\in\rect(\+Y^{(s)}, Y^{(t)})$. 
	
	By our choice of $\psi$, it holds that $\grad^2\psi(\xi^{(s)})\mgne 0$ for any $\xi^{(s)}\in\rect(\+Y^{(s)}, Y^{(t)})$.  Therefore, combining \Cref{eqn:gradY-bound,eqn:gradY-Hessian-bound}, we have 
	\[
	\+Y^{(s)}\ge Y^{(t)}-(s-t)\cdot \grad^{-2}\psi(\xi^{(s)})\cdot \wh{L}^{(t)}.
	\]
	Similar argument yields that
	\[
	\+X_{k}^{(s)}\ge X_{k}^{(t)}-(s-t)\cdot \grad^{-2}\phi_{k}(\zeta_k^{(s)})\cdot \hat{\ell}_{k}^{(t)}
	\] 
	for some $\zeta_k^{(s)}\in\rect(\+X_{k}^{(s)}, X_{k}^{(t)})$.
	
	Therefore for any $k\in [K]$, $j\in [m_k]$ and any $s\in [t,t+1)$, we can bound the difference between $Z^{(t)}(k,j)$ and $\+Z^{(s)}(k,j)$:
	\begin{align*}
		&\phantom{{}={}}Z^{(t)}(k,j) - \+Z^{(s)}(k,j) =Y^{(t)}(k)\cdot X^{(t)}_k(j) - \+Y^{(s)}(k)\cdot \+X^{(s)}_k(j)\\
		&\le Y^{(t)}(k)\cdot X^{(t)}_k(j) - \tp{Y^{(t)}(k)-(s-t)\cdot \left[\grad^{-2}\psi(\xi^{(s)})\cdot \wh{L}^{(t)}\right](k)}\cdot\tp{X_{k}^{(t)}(j)-(s-t)\cdot \left[\grad^{-2}\phi_{k}(\zeta_k^{(s)})\cdot \hat{\ell}_{k}^{(t)}\right](j)}\\
		&=-(s-t)^2\cdot \left[\grad^{-2}\psi(\xi^{(s)})\cdot\wh L^{(t)}\right](k)\cdot\left[\grad^{-2}\phi_k(\zeta^{(s)}_k)\cdot \hat\ell^{(t)}_k\right](j) +(s-t)\cdot X^{(t)}_k(j)\cdot\left[\grad^{-2}\psi(\xi^{(s)})\cdot\wh L^{(t)}\right](k)\\
		&\phantom{{}={}}+(s-t)\cdot Y^{(t)}(k)\cdot\left[\grad^{-2}\phi_k(\zeta^{(s)}_k)\cdot\hat\ell^{(t)}_k\right] (j)\\
		&\le  (s-t)\cdot X^{(t)}_k(j)\cdot\left[\grad^{-2}\psi(\xi^{(s)})\cdot\wh L^{(t)}\right](k)
		+(s-t)\cdot Y^{(t)}(k)\cdot\left[\grad^{-2}\phi_k(\zeta^{(s)}_k)\cdot\hat\ell^{(t)}_k\right] (j)
	\end{align*}
	for some $\xi^{(s)}\in\rect(\+Y^{(s)}, Y^{(t)})$ and $\zeta_k^{(s)}\in\rect(\+X_{k}^{(s)}, X_{k}^{(t)})$.
	
	\smallskip
	We are now ready to bound the gap between $R_a(T)$ and $\@R_a(T)$:
	\begin{align*}
		R_a(T)-\@R_a(T)
		&=\sum_{t=0}^{T-1}\E{\int_{t}^{t+1}\inner{Z^{(t)}-\@Z^{(s)}}{\hat{\ell}^{(t)}}}\\
		&\le \underbrace{\sum_{t=0}^{T-1}\E{\int_t^{t+1} (s-t)\tp{\sum_{k\in [K]}\sum_{j\in [m_k]} X^{(t)}_k(j)\cdot\sup_{\xi\in\rect(Y^{(t)},\ol Y^{(t+1)})}\left[\grad^{-2}\psi(\xi)\cdot\wh L^{(t)}\right](k)}\cdot\hat\ell^{(t)}_k(j)\;\dd s}}_{(A)}\\
		&\quad+\underbrace{\sum_{t=0}^{T-1}\E{\int_t^{t+1} (s-t)\tp{\sum_{k\in [K]}\sum_{j\in [m_k]}Y^{(t)}(k)\cdot\sup_{\zeta_k\in\rect(X^{(t)}_k,\ol X^{(t+1)}_k)}\left[\grad^{-2}\phi_k(\zeta_k)\cdot\hat\ell^{(t)}_k\right](j)}\cdot\hat\ell^{(t)}_k(j)\;\dd s}}_{(B)}.
	\end{align*}
	Note that in both expressions (A) and (B) above, only the term $(s-t)$ depend on $s$. So we can integrate and obtain:
	\begin{align}
		(A)\label{eqn:A-norm}
		&=\frac{1}{2}\sum_{t=0}^{T-1}\E{\tp{\sum_{k\in [K]}\sum_{j\in [m_k]} X^{(t)}_k(j)\cdot\sup_{\xi\in\rect(Y^{(t)},\ol Y^{(t+1)})}\left[\grad^{-2}\psi(\xi)\cdot\wh L^{(t)}\right](k)}\cdot\hat\ell^{(t)}_k(j)}\\\notag
		&=\frac{1}{2}\sum_{t=0}^{T-1}\E{\sum_{k\in [K]} \sup_{\xi\in\rect(Y^{(t)},\ol Y^{(t+1)})}\left[\grad^{-2}\psi(\xi)\cdot\wh L^{(t)}\right](k) \cdot \tp{\sum_{j\in[m_k]} X^{(t)}_k(j)\cdot \hat\ell^{(t)}_k(j)}}\\\notag
		&=\frac{1}{2}\sum_{t=0}^{T-1}\E{\sum_{k\in [K]} \sup_{\xi\in\rect(Y^{(t)},\ol Y^{(t+1)})}\left[\grad^{-2}\psi(\xi)\cdot\wh L^{(t)}\right](k) \cdot \wh L^{(t)}(k)}\\\notag
		&=\frac{1}{2}\sum_{t=0}^{T-1} \E{\sup_{\xi\in\rect(Y^{(t)},\ol Y^{(t+1)})} \|\wh L^{(t)}\|^2_{\grad^{-2}\psi(\xi)}}.
	\end{align}
	Similarly, 
	\begin{align}
		(B)\label{eqn:B-norm}
		&=\frac{1}{2}\sum_{t=0}^{T-1}\E{\tp{\sum_{k\in [K]}\sum_{j\in [m_k]}Y^{(t)}(k)\cdot\sup_{\zeta_k\in\rect(X^{(t)}_k,\ol X^{(t+1)}_k)}\left[\grad^{-2}\phi_k(\zeta_k)\cdot\hat\ell^{(t)}_k\right](j)}\cdot\hat\ell^{(t)}_k(j)}\\\notag
		&=\frac{1}{2}\sum_{t=0}^{T-1}\E{\sum_{k\in [K]} Y^{(t)}(k)\cdot \sup_{\zeta_k\in\rect(X^{(t)}_k,\ol X^{(t+1)}_k)} \|\hat\ell^{(t)}_k\|^2_{\grad^{-2}\phi_k(\zeta_k)}}.
	\end{align}
	
	Combining \Cref{eqn:A-norm,eqn:B-norm}, we have
	\begin{equation}\label{eqn:gap-norm-bound}
		R_a(T)-\@R_a(T)
		\le \frac{1}{2}\sum_{t=0}^{T-1} \E{\sup_{\xi\in\rect(Y^{(t)},\ol Y^{(t+1)})} \|\wh L^{(t)}\|^2_{\grad^{-2}\psi(\xi)} + \sum_{k\in [K]} Y^{(t)}(k)\cdot \sup_{\zeta_k\in\rect(X^{(t)}_k,\ol X^{(t+1)}_k)} \|\hat\ell^{(t)}_k\|^2_{\grad^{-2}\phi_k(\zeta_k)}}.
	\end{equation}
	
\end{proof}

%
%

If we apply the ``regret decomposition theorem'' in~\cite{HZ22} and use the standard OSMD bound for each stage, we will get the term 
\begin{equation}\label{eqn:regret-HZ}
	\sup_{\zeta_{k^*}\in\rect(X^{(t)}_{k^*},\ol X^{(t+1)}_{k^*})} \|\hat\ell^{(t)}_{k^*}\|^2_{\grad^{-2}\phi_{k^*}(\zeta_{k^*})}
\end{equation}
where $k^*$ is the index of the group containing the optimal arm instead of the term
\[
	\sum_{k\in [K]} Y^{(t)}(k)\cdot \sup_{\zeta_k\in\rect(X^{(t)}_k,\ol X^{(t+1)}_k)} \|\hat\ell^{(t)}_k\|^2_{\grad^{-2}\phi_k(\zeta_k)}
\]
in \cref{eqn:gap-norm-bound}. The new $Y^{(t)}(k)$ term is crucial to our optimal regret bound since it cancels a $Y^{(t)}(k)$ term hidden in the denominator of $\|\hat\ell^{(t)}_k\|^2_{\grad^{-2}\phi_k(\zeta_k)}$. This will be clear in \Cref{sec:proof}.


\subsubsection{The Regret of Algorithm~\ref{algo:main-algo}}\label{sec:proof}
Note that the regret of Algorithm~\ref{algo:main-algo} is composed of the two parts in \Cref{lem:continuous-regret} and \Cref{lem:compare}. In this section, we will prove \Cref{thm:ub_clique} by providing more specific bounds for the terms in these two lemmas.

\begin{proof}[Proof of \Cref{thm:ub_clique}]
By definition of Bregman divergence, 
$$B_\psi(\*e_k, Y^{(0)}) =\psi(\*e_k)-\psi(Y^{(0)}) -\inner{\nabla\psi(Y^{(0)})}{\*e_k-Y^{(0)}}.$$
Since we initialize  $Y^{(0)}=\argmin_{b\in \Delta_{K-1}}\psi(b)$,   $Y^{(0)}(k)=\frac{1}{K}$  for $k\in [K]$ and $\inner{\nabla\psi(Y^{(0)})}{\*e_k-Y^{(0)}}\ge 0$ follows the first-order optimality condition for $Y^{(0)}$. Thus 
$$B_\psi(\*e_k, Y^{(0)}) \leq \psi(\*e_k)-\psi(Y^{(0)}) =\frac{-2+ 2\sqrt{K}}{\eta}\leq \frac{2\sqrt{K}}{\eta}.$$
Similarly we have $X_k^{(0)}(j)=\frac{1}{m_k}$ for $j\in [m_k]$ and 
$$B_{\phi_k}(\*e_j,X^{(0)}_k)\leq \phi_k(\*e_j)-\phi_k(X^{(0)}_k)=\frac{\log m_k}{\eta_k}.$$
Therefore \begin{equation}\label{eqn:continuous}
	\@R_a(T)\leq \frac{2\sqrt{K}}{\eta}+\frac{\log m_k}{\eta_k}.
\end{equation}

Recall that $A_t=(k_t,j_t)$ is the arm pulled by the algorithm at round $t$. Now we plug our estimator $ \hat\ell_k^{(t)}(j)=\frac{\1{k_t=k}}{Y^{(t)}({k})}\ell_k^{(t)}(j)$ and $\nabla^2\psi(\xi)=\text{diag}\tuple{\frac{1}{2\eta \xi(1)^{3/2}},\frac{1}{2\eta\xi(2)^{3/2}},\cdots, \frac{1}{2\eta\xi(K)^{3/2}} }$ into the first term on the RHS of \Cref{lem:compare}.
\begin{align*}
	\E{\sup_{\xi\in\rect(Y^{(t)},\ol Y^{(t+1)})} \|\wh L^{(t)}\|^2_{\grad^{-2}\psi(\xi)} } &= 2\eta \E{ \sup_{\xi\in\rect(Y^{(t)}, \ol{Y}^{(t+1)}) }\sum_{k\in [K]}\xi(k)^{3/2}\cdot \tuple{ \frac{\1{k_t = k}}{Y^{(t)}(k)}\sum_{j\in[m_k]}\ell_k^{(t)}(j) {X}_k^{(t)}(j) }^2 }  \\
	&\overset{(a)}{\leq}2\eta \E{ \sum_{k\in [K]}\tuple{Y^{(t)}(k)}^{3/2}\cdot \tuple{ \frac{\1{k_t=k}}{Y^{(t)}(k)}\sum_{j\in[m_k]}\ell_k^{(t)}(j) {X}_k^{(t)}(j) }^2 }\\
	&\overset{(b)}{\le} 2\eta\E{ \E{ \sum_{k\in [K]}  \frac{\1{k_t=k}}{\sqrt{Y^{(t)}(k)}} }\ \Bigg|\ Y^{(t)}}\\
	&= 2 \eta \sum_{k=1}^K\E{\sqrt{Y^{(t)}(k)} }\overset{(c)}{\leq}  2\eta \sum_{k=1}^K\sqrt{\E{Y^{(t)}(k)}} \leq 2\eta\sqrt{K}.
\end{align*}
In the calculation above: $(a)$ follows from $\ol{Y}^{(t+1)}(k)\leq Y^{(t)}(k)$, $(b)$ is due to $\sum_{j\in[m_k]}\ell_k^{(t)}(j) {X}_k^{(t)}(j)\in[0,1]$, and $(c)$ is due to Jensen's inequality.

Similarly we have for the second term with $\nabla^2\phi_k(\zeta_k)=\text{diag}\tuple{\frac{1}{\eta_k \zeta_k(1)},\frac{1}{\eta_k \zeta_k(2)},\cdots, \frac{1}{\eta_k \zeta_k(m_k)}  }$
\begin{align*}
	&\phantom{{}={}}\E{ \sum_{k\in [K]} Y^{(t)}(k)\cdot \sup_{\zeta_k\in\rect(X^{(t)}_k,\ol X^{(t+1)}_k)} \|\hat\ell^{(t)}_k\|^2_{\grad^{-2}\phi_k(\zeta_k)}}\\
	&=\E{\sum_{k\in[K]}\eta_kY^{(t)}(k)\cdot \sup_{\zeta_k\in \rect(X_k^{(t)}, \ol{X}_k^{(t+1)})}  \sum_{j\in[m_k]} \zeta_k(j)\cdot\tuple{\frac{\1{k_t = k}}{Y^{(t)}(k)}\ell_k^{(t)}(j)}^2} \\
	&\overset{(d)}{\leq} \E{\sum_{k\in[K]}\eta_kY^{(t)}(k)\cdot  \sum_{j\in[m_k]} X^{(t)}_k(j)\cdot\tuple{\frac{\1{k_t=k}}{Y^{(t)}(k)}\ell_k^{(t)}(j)}^2} \\
	&\overset{(e)}{\leq}\E{\E{\sum_{k\in[K]}\eta_k\cdot  \sum_{j\in[m_k]} X^{(t)}_k(j)\cdot \frac{\1{k_t= k}}{Y^{(t)}(k)}} \ \Bigg| \ Y^{(t)}(k)}\\
	&= \sum_{k\in[K]}\eta_k  \sum_{j\in[m_k]}X_k^{(t)}(j) = \sum_{k\in[K]}\eta_k.
\end{align*}
In the calculation above: $(d)$ follows from $\ol{X}_k^{(t+1)}(j)\leq X_k^{(t)}(j)$ and $(e)$ is due to $\ell_k^{(t)}(j)\in[0,1]$.

Hence, summing up above two terms from $0$ to $T-1$, we obtain
\begin{equation}\label{eqn:compare}
	R_a(T)-\@R_a(T)\leq \eta \sqrt{K}T +\frac{1}{2}T \sum_{k\in[K]}\eta_k.
\end{equation}
Combining \Cref{eqn:continuous,eqn:compare} and choosing $\eta=\frac{1}{\sqrt{T}}$ and  $\eta_k=\frac{\log \tp{m_k+1}}{\sqrt{T\sum_{k=1}^K\log (m_k+1)}}$, we obtain for any fixed arm $a$,
\[
	R_a(T)\leq \frac{2\sqrt{K}}{\eta}+\frac{\log m_{k}}{\eta_{k}} + \frac{T}{2}\sum_{k\in [K]}\eta_k+\eta T\sqrt{K}\leq O\tuple{\sqrt{T\sum_{k=1}^K\log (m_k+1)}}.
\]
\end{proof}

\subsection{A Reduction from \BAI to \MAB}\label{subsec:bai_ub}
In this section, we prove an upper bound of $O\tuple{\sum_{k=1}^K\frac{\log (m_k+1)}{\eps^2}}$ for $\*m$-\BAI. We achieve this by constructing a PAC algorithm for $\*m$-\BAI from an algorithm for $\*m$-\MAB through the following lemma.

Let $r(T,\vec{L})$ be a real valued function with the time horizon $T$ and loss sequence $\vec{L}=\tp{\ell^{(1)},\dots, \ell^{(T)}}$ as its input. Let $\@H$ be a \BAI instance. With fixed $T>0$, we use $\E[\@H]{r(T,\vec{L})}$ to denote the expectation of $r(T,\vec{L})$ where $\ell^{(t)}$ in $\vec L$ is drawn from $\@H$ independently for every $t\in [T]$. Let $\+H$ be a set of \BAI instances.
\begin{lemma}\label{lem: reduction}
	Let $\+A$ be an algorithm for $\*m$-\MAB with regret $R_{a^*}(T,\+A,\vec{L})\leq r(T,\vec{L})$ for every time horizon $T$ and every loss sequence $\vec{L}$. Then there exists an $(\eps,0.05)$-PAC algorithm $\+A'$ for $\*m$-\BAI that terminates in $T^*$ rounds where $T^*$ is the solution of the equation 
	\[
		T^* = \frac{2500\cdot \max_{\vec{L}}r(T^*,\vec{L})}{\eps}.
	\]

	Moreover, if we only care about identifying an $\eps$-optimal arm with probability $0.95$ when the input is chosen from a known family $\+H$, we can construct an algorithm solving this problem that terminates in $T^*_{\+H}$ rounds where $T^*_{\+H}$ is the solution of the equation 
	\[
		T^*_{\+H} = \frac{2500\cdot \max_{\@H\in \+H}\E[\@H]{r(T^*_{\+H},\vec{L})}}{\eps}.
	\]
\end{lemma}
\begin{proof}
	Given an instance $\@H$ of $\*m$-\BAI, we run $\+A$ for $T^*$ rounds. Let $T_i$ be the number of times that the arm $i$ has been pulled, i.e., $T_i=\sum_{t=0}^{T^*-1}\bb{1}[A_t=i]$. Let $\ol{Z}=\tp{\ol{Z}_1,\ol{Z}_2,\dots,\ol{Z}_N} = \tp{\frac{T_1}{T^*}, \frac{T_2}{T^*}, \dots, \frac{T_N}{T^*}}$ be a distribution on $N$ arms. We construct $\+A'$ by simply sampling from $\ol{Z}=\tp{\frac{T_1}{T^*}, \frac{T_2}{T^*}, \dots, \frac{T_N}{T^*}}$ and outputting the result.

	Recall that $p_i$ is the mean of the $i$-th arm in $\@H$ and arm $a^*$ is the one with the minimum mean.  Define the gap vector $\Delta=(p_1-p_{a^*},\cdots,p_N-p_{a^*})$. Note that $\ol{Z}$ is a random vector and define conditional expected regret $R(\ol{Z})= \inner{\Delta}{\ol{Z}}\cdot T^*$ given $\ol{Z}$. Thus the expected regret $\E[\ol{Z}]{R(\ol{Z})}\leq \max_{\vec{L}}r(T^*,\vec{L})$. By Markov's inequality, $R(\ol{Z})\leq 100\max_{\vec{L}}r(T^*,\vec{L})$ with probability at least $0.99$. Now we only consider $\ol{Z}$ conditioned on $R(\ol{Z})\leq 100\max_{\vec{L}}r(T^*,\vec{L})$.  Let $B\subseteq [N]$ denote the ``bad set'' which contains arms that are not $\eps$-optimal. Then $\eps T^*\sum_{i\in B} \ol{Z}_i\leq 100 \max_{\vec{L}}r(T^*,\vec{L})$. Note that $T^* = \frac{2500\cdot \max_{\vec{L}}r(T^*,\vec{L})}{\eps}$. Therefore $\sum_{i\in B} \ol{Z}_i\leq 0.04$. In total, this algorithm will make a mistake with probability no more than $0.05$ by the union bound.

	When we only care about the input instances chosen from $\+H$, we run $\+A$ for $T^*_{\+H}$ rounds and similarly, we output an arm drawn from $ \tp{\frac{T_1}{T^*_{\+H}}, \frac{T_2}{T^*_{\+H}}, \dots, \frac{T_N}{T^*_{\+H}}}$. It is easy to verify via the same arguments that this algorithm can output an $\eps$-optimal arm with probability $0.95$ when the input is chosen from $\+H$.
\end{proof}

Then we can use the Algorithm~\ref{algo:main-algo} and \Cref{thm:ub_clique} to give an upper bound for $\*m$-\BAI. 

\begin{proof}[Proof of \Cref{thm:ub-m-BAI}]
	We use Algorithm~\ref{algo:main-algo} to construct an $\tp{\eps,0.05}$-PAC algorithm for $\*m$-\BAI as described in Lemma~\ref{lem: reduction}. Since the regret satisfies $R(T)\leq c{\sqrt{T\sum_{k=1}^K\log (1+m_k)}}$ for some constant $c$ on every loss sequence by \Cref{thm:ub-m-MAB}, running Algorithm~\ref{algo:main-algo} with $T^*=\frac{(2500c)^2\sum_{k=1}^K\log (1+m_k)}{\eps^2}$, we can get an $(\eps,0.05)$-PAC algorithm which always terminates in $O\tuple{\sum_{k=1}^K\frac{\log (m_k+1)}{\eps^2}}$ rounds.
\end{proof}

\subsection{The Strongly Observable Graph with Self-loops} \label{subsec:strongly-ub}

We can generalize our results to any strongly observable graph $G=(V,E)$ with each vertex owning a self-loop. Assume $G$ contains a $(V_1,\dots,V_K)$-clique cover. We construct a new graph $G^\prime=(V,E^\prime)$ by ignoring the edges between any two distinct cliques. It is clear that $R^*(G,T)\leq R^*(G^\prime,T)$. Then we can prove \Cref{cor:ub-graph-bandit} by directly applying Algorithm~\ref{algo:main-algo} with feedback graph $G'$. This proves \Cref{cor:ub-graph-bandit}, which asserts that
\[
	R^*(G,T) = O\tp{\sqrt{T\cdot\sum_{k=1}^K\log(m_k+1)}}.
\]

Although we assume that each vertex contains a self-loop for the sake of simplicity, we note that our algorithm can still be applied to strongly observable graphs that have some vertices without self-loops. In such cases, we can incorporate an additional exploration term into our algorithm, and a similar analysis to that in \Cref{sec:analysis} still works. 

There have been several works using the clique cover as the parameter to bound the minimax regret of graph bandit. For example,~\cite{EK21} applies FTRL algorithm with a carefully designed potential function which combines the Tsallis entropy with negative entropy. It achieves a regret of $\tp{\log T}^{O(1)}\cdot O\tp{\sqrt{KT}}$. Our new bound takes into account the size of each clique and is always superior.

\section{Lower Bounds for \texorpdfstring{$\*m$}~-\BAI}\label{sec: BAI lb}


Let $\+A$ be an algorithm for $\*m$-\BAI where $\*m=(m_1,\dots,m_K)$ is a vector. Given an instance of $\*m$-\BAI, we use $T$ to denote the number of rounds the algorithm $\+A$ proceeds. Recall that for every group $k\in [K]$ and $j\in [m_k]$, we use $T_{(k,j)}$ to denote the number of times that the arm $(k,j)$ has been pulled. For every $k\in [K]$, let $T^{(k)} = \sum_{j\in [m_k]} T_{(k,j)}$ be the number of rounds the arms in the $k$-th group have been pulled. We also use $N_{(k,j)}$ to denote the number of times the arm $(k,j)$ has been {observed}. Clearly $N_{(k,j)} = T^{(k)}$.

In the following part, we only consider stochastic environment. That is, $\ell^{(t)}$ is independently drawn from the same distribution for each $t\in \bb N$. Therefore, we omit the superscript $(t)$ and only use $\ell(i)$ or $\ell_k(j)$ to denote the one-round loss of arm $i$ or arm $(k,j)$ respectively when the information is clear from the context. 

In \Cref{subsec: BAI-lb}, we lower bound the number of rounds for a PAC algorithm on a specific $\*m$-\BAI instance with $\*m=(m)$ and then prove the result for $\*m$-\BAI in \Cref{subsec: block BAI-lb}. We then use these results to prove a regret lower bound for $\*m$-\MAB and bandit problems with general feedback graphs in \Cref{sec: regret lb}. 

\subsection{An Instance-Specific Lower Bound for \texorpdfstring{$(m)$}~-\BAI}\label{subsec: BAI-lb}


In this section, we study the number of rounds required for $(m)$-\BAI in an $(\eps,0.05)$-PAC algorithm. In this setting, the pull of any arm can observe the losses of all arms. We will establish a lower bound for a specified instance, namely the one where all arms follow $\!{Ber}(\frac{1}{2})$. This is key to our lower bound later. 


We focus on instances of $(m)$-\BAI where each arm is Bernoulli. As a result, each instance can be specified by a vector $\tp{p_1,\dots,p_{m-1},p_{m}}\in\bb R^{m}$ meaning the loss of arm $i$ follows $\!{Ber}(p_i)$ in each round \emph{independently}.

Let $\eps\in \tp{0,\frac{1}{2}}$. In the following context, when we denote an instance as $\@{H}^{\*m}$, the superscript $\*m$ indicates that it is an $\*m$-\BAI instance. Consider the following $m+1$ $(m)$-\BAI instances $\set{\@{H}_j^{(m)}}_{j\in[m]\cup \set{0}}$:
\begin{itemize}
	\item The instance $\@{H}_0^{(m)}$ is $\tp{\frac{1}{2}, \frac{1}{2},\frac{1}{2}, \cdots, \frac{1}{2}}$. That is, $p_i=\frac{1}{2}$ for every $i\in[m]$ in $\@{H}_0^{(m)}$;
	\item For $j\in[m]$, 
	\[
		\@{H}_j^{(m)} = \tp{\frac{1}{2},\frac{1}{2},\cdots,\frac{1}{2}, \mathop{\frac{1}{2}-\eps}\limits_{\substack{\uparrow \\ \mbox{the }j\mbox{-th arm}}}, \frac{1}{2}, \cdots, \frac{1}{2}};
	\]
	that is, the instance satisfies $p_j=\frac{1}{2}-\eps$ and $p_i=\frac{1}{2}$ for every $i\neq j$.
\end{itemize}

We say an algorithm $\+A$ distinguishes $\set{\@{H}_j^{(m)}}_{j\in[m]\cup\set{0}}$ with probability $p$ if 
\[
\Pr{\+A \mbox{ outputs }j\mid \mbox{the input instance is }\@H^{(m)}_j}	\ge p,
\]
and the output can be arbitrary among $\set{0,1,\dots m}$ when the input is not in $\set{\@{H}_j^{(m)}}_{j\in[m]\cup\set{0}}$.

The main result of this section is 

\begin{lemma}\label{lem:block-BAI-lb}
	Let $\+A$ be an $(\eps,0.05)$-PAC algorithm. Assume $m\geq 2$. There exists a universal constant $c_1>0$ such that $\+A$ terminates on $\@{H}^{(m)}_0$ after at least $\frac{c_1}{\eps^2}\log (m+1)$ rounds in expectation.
\end{lemma}

We will prove the lemma in \Cref{ssubsec: gau2ber} via a reduction from a lower bound for \emph{Gaussian arms} established in \Cref{ssubsec: gaussian}.

\subsubsection{The Gaussian Arms}\label{ssubsec: gaussian}

In this section, we relax the constraint on the range of each arm's loss and allow the losses to be arbitrary real numbers. Let $\eps\in\tp{0,\frac{1}{2}}$ and $\sigma\in \tp{\frac{1}{2\sqrt{2\pi}},\frac{1}{\sqrt{2\pi}}}$. We construct $m+1$ instances $\set{\@{N}_j}_{j\in\set{0}\cup [m]}$ with Gaussian distributions:
\begin{itemize}
	\item In the instance $\@{N}_0$, for each $i\in[m]$, $\ell(i)$ is independently drawn from a Gaussian distribution $\+N(0,\sigma^2)$;
	\item In the instance $\@{N}_j$ for $j\in[m]$, $\ell(j)\sim \+N(-\eps,\sigma^2)$ and $\ell(i)\sim \+N(0,\sigma^2)$ for each $i\neq j$ and $i\in[m]$ independently. 
\end{itemize}



\begin{lemma}[Bretagnolle-Huber inequality, see e.g.~\cite{LS20}]\label{thm: huber}
	Let $\*{P}_1$ and $\*{P}_2$ be two probability measures on the same measurable space $\tp{\Omega,\+F}$, and let $E\in \+F$ be an arbitrary event. Then 
	\[
		\*P_1[E]+\*P_2[\ol E]\geq \frac{1}{2}e^{-\DKL\tp{\*{P}_1, \*{P}_2}}
	\]
\end{lemma}

Let $\@{N}_{\!{mix}}$ be the mixture of $\set{\@{N}_{j}}_{j\in[m]}$ meaning that the environment chooses $k$ from $[m]$ uniformly at random and generates losses according to $\@{N}_k$ in the following \BAI game. Let $\+A$ be an algorithm distinguishing $\set{\@{N}_{j}}_{j\in[m]\cup \set{0}}$. Let $\Omega$ be the set of all possible outcomes during the first $t^*$ rounds, including the samples according to the input distribution and the output of $\+A$ (if $\+A$ does not terminate after the $t^*$-th round, we assume its output is $-1$). Note that if the algorithm terminates in $t'<t^*$ rounds, we can always add $t^*-t'$ virtual rounds so that it still produces a certain loss sequence in $\bb R^{m\times t^*}$. 

As a result, each outcome $\omega\in \Omega$ can be viewed as a pair $\omega = (w, x)$ where $w\in \bb R^{m\times t^*}$ is the loss sequence and $x\in\set{-1,0,1,\dots,m}$ indicates the output of $\+A$. Thus $\Omega = W\times \set{-1,0,1,\dots,m}$ where $W=\bb R^{m\times t^*}$.

To ease the proof below, we slightly change $\+A$'s output: if the original output is $x\in \set{-1,0,\dots,m}$, we instead output a uniform real in $[x,x+1)$. Therefore, we can let $\Omega=W\times X$ where $W=\bb R^{m\times t^*}$ and $X=\bb R$. The benefit of doing so is that we can let $\+F$ be the Borel sets in $\Omega$ which is convenient to work with. Clearly it is sufficient to establish lower bounds for the algorithms after the change.


For any instance $\@{H}^{(m)}$, let $\*{P}_{\@{H}^{(m)}}$ be the measure of outcomes of $\+A$ in $t^*$ rounds with input instance $\@{H}^{(m)}$ and $\*p_{\@{H}^{(m)}}$ be the corresponding probability density function (PDF). Then $\*{P}_{\@{N}_0}$ and $\*{P}_{\@{N}_{\!{mix}}}$ are two probability measures on $(\Omega,\+F)$ and $\*p_{\@{N}_{\!{mix}}}(\omega)=\frac{1}{m}\sum_{j\in[m]}\*p_{\@{N}_j}(\omega)$ for any $\omega=(w,x)\in \Omega=\bb R^{m\times t^*+1}$. We also let $\*p^W_{\@{H}^{(m)}}$ be the PDF of the samples during the first $t^*$ rounds according to the input $\@{H}^{(m)}$ and $\*p^X_{\@{H}^{(m)}}$ be the PDF of $\+A$'s output. Furthermore, we let $\*p^{X|W}_{\@{H}^{(m)}}$ to be the conditional density function of $X$ given $W$. By definition, we have $\*p^{X|W}_{\@{H}^{(m)}}(x|w) = \frac{\*p_{\@{H}^{(m)}}(\omega)}{\*p^{W}_{\@{H}^{(m)}}(w)}$.


\begin{lemma}\label{lem: gaussian KL}
	\[
		\DKL\tp{\*{P}_{\@{N}_{\!{mix}}},\*{P}_{\@{N}_0}}\leq \log \frac{{m-1+\exp\tuple{\frac{\eps^2t^*}{\sigma^2}}}}{m}.
	\]
\end{lemma}
\begin{proof}
	For any $\omega=(w,x) \in \Omega$, let $w_{j,t}$ denote the $(j,t)^\text{th}$ entry of the matrix $w$ for every $j\in[m]$ and $t\in[t^*]$. That is, $w_{j,t}=\ell^{(t)}(j)$, which is the loss of arm $j$ in the $t$-th round. Then for each $i\in[m]$,
	\[
		\*p^W_{\@{N}_i}(w)=\tp{2\pi\sigma^2}^{-\frac{mt^*}{2}}\exp\tp{-\frac{\sum_{t\in [t^*]}\tp{(w_{i,t}+\eps)^2+\sum_{j\neq i}w_{j,t}^2}}{2\sigma^2}}
	\] and 
	\[
		\*p^W_{\@{N}_0}(w)=\tp{2\pi\sigma^2}^{-\frac{mt^*}{2}}\exp\tp{-\frac{\sum_{t\in [t^*],j\in[m]}w_{j,t}^2}{2\sigma^2}}.
	\]
	Therefore we have
	\begin{align*}
		\frac{\*p_{\@{N}_i}(\omega)}{\*p_{\@{N}_0}(\omega)}= \frac{\*p^W_{\@{N}_i}(w)}{\*p^W_{\@{N}_0}(w)} &=\frac{\tp{2\pi\sigma^2}^{-\frac{mt^*}{2}}\exp\tp{-\frac{\sum_{t\in [t^*]}\tp{(w_{i,t}+\eps)^2+\sum_{j\neq i}w_{j,t}^2}}{2\sigma^2}}}{\tp{2\pi\sigma^2}^{-\frac{mt^*}{2}}\exp\tp{-\frac{\sum_{t\in [t^*],j\in[m]}w_{j,t}^2}{2\sigma^2}}} \\
		&=\exp\tp{-\frac{\eps^2t^*+2\eps\sum_{t\in[t^*]}w_{i,t}}{2\sigma^2}}.
	\end{align*}
	From Jensen's inequality, we have 
	\begin{align*}
		\DKL\tp{\*{P}_{\@{N}_{\!{mix}}},\*{P}_{\@{N}_0}}&= \int_{\Omega} \log\frac{\*p_{\@{N}_{\!{mix}}}(\omega)}{\*p_{\@{N}_0}(\omega)}\d {\*P_{\@{N}_{\!{mix}}}(\omega)} \leq \log \int_{\Omega} \frac{\*p_{\@{N}_{\!{mix}}}(\omega)}{\*p_{\@{N}_0}(\omega)}\d {\*P_{\@{N}_{\!{mix}}}(\omega)} \\
		&= \log \int_{\Omega} \frac{1}{m}\sum_{j\in[m]}\*p_{\@{N}_j}(\omega)\frac{\frac{1}{m}\sum_{i\in[m]}\*p_{\@{N}_j}(\omega)}{\*p_{\@{N}_0}(\omega)}\d \omega.
	\end{align*}
	Note that for $\omega=(w,x)$,  For $i,j\in[m]$ and $i\neq j$,
	\begin{align*}
		\int_{\Omega}\*p_{\@{N}_i}(\omega)\frac{\*p_{\@{N}_j}(\omega)}{\*p_{\@{N}_0}(\omega)}\d \omega &=\int_{W}\int_{X}\*p^W_{\@{N}_i}(w)\cdot \*p^{X|W}_{\@{N}_i}(x|w) \frac{\*p^W_{\@{N}_j}(w)}{\*p^W_{\@{N}_0}(w)}\dd x\dd w \\
		&= \int_{W}\*p^W_{\@{N}_i}(w) \frac{\*p^W_{\@{N}_j}(w)}{\*p^W_{\@{N}_0}(w)}\dd w \\
		&= \tp{2\pi\sigma^2}^{-\frac{mt^*}{2}}\cdot \int_{\Omega}\exp\tp{-\frac{\sum_{t\in[t^*]}\tp{(w_{i,t}+\eps)^2+(w_{j,t}+\eps)^2}+\sum_{\substack{j'\neq i \\ j'\neq j}}w_{j',t}^2}{2\sigma^2}}\d w = 1.
	\end{align*}
	For $i\in[m]$, 
	\begin{align*}
		\int_{\Omega}\*p_{\@{N}_i}(\omega)\frac{\*p_{\@{N}_i}(\omega)}{\*p_{\@{N}_0}(\omega)}\d \omega &= \int_{W}\int_{X}\*p^W_{\@{N}_i}(w)\cdot \*p^{X|W}_{\@{N}_i}(x|w) \frac{\*p^W_{\@{N}_i}(w)}{\*p^W_{\@{N}_0}(w)}\dd x\dd w \\
		&= \int_{W}\*p^W_{\@{N}_i}(w) \frac{\*p^W_{\@{N}_i}(w)}{\*p^W_{\@{N}_0}(w)}\dd w \\
		&= \tp{2\pi\sigma^2}^{-\frac{mt^*}{2}}\cdot \int_{\Omega}\exp\tp{ -\frac{\sum_{t\in[t^*]}\tp{(w_{i,t}+2\eps)^2+\sum_{j^\prime\neq i}w_{j',t}^2}-2\eps^2t^*}{2\sigma^2}}\d w = \exp\tp{\frac{\eps^2t^*}{\sigma^2}}.
	\end{align*}

	Therefore, combining the equations above, we get
	\begin{align*}
		\int_{\Omega} \frac{1}{m}\sum_{j\in[m]}\*p_{\@{N}_j}(\omega)\frac{\frac{1}{m}\sum_{i\in[m]}\*p_{\@{N}_i}(\omega)}{\*p_{\@{N}_0}(\omega)}\d \omega &= \frac{1}{m^2}\sum_{i,j\in[m]}\int_{\Omega} \*p_{\@{N}_i}(\omega)\frac{\*p_{\@{N}_j}(\omega)}{\*p_{\@{N}_0}(\omega)} \d \omega\\
		&= \frac{m(m-1) + m\cdot \exp\tp{\frac{\eps^2t^*}{\sigma^2}}}{m^2} = \frac{m-1+\exp\tp{\frac{\eps^2t^*}{\sigma^2}}}{m},
	\end{align*}
	where the first equality follows from Fubini's theorem. This indicates that $\DKL\tp{\*{P}_{\@{N}_{\!{mix}}},\*{P}_{\@{N}_0}}\leq \log \frac{m-1+\exp\tp{\frac{\eps^2t^*}{\sigma^2}}}{m}$.
\end{proof}

Let $t^*=\frac{c_0\log (m+1)}{\eps^2}$, where  $c_0\leq\sigma^2$ is a universal constant. We have the following lemma to bound $\Pr[{\@{N}_0}]{T\geq t^*}$.  Here the randomness comes from the algorithm and environment when the input instance is $\@{N}_0$.
\begin{lemma}\label{lem: gaussian-lb}
	For any algorithm distinguishing $\set{\@{N}_j}_{j\in[m]\cup\set{0}}$ with probability $0.925$, we have $\Pr[{\@{N}_0}]{T\geq t^*}\geq 0.1$.
\end{lemma}
\begin{proof}

	Let $\+A$ be an algorithm that can distinguish $\set{\@{N}_j}_{j\in[m]\cup\set{0}}$ with probability $0.925$. Let $E$ be the event that $\+A$ terminates within $t^*$ rounds and gives answer $\@{N}_0$. Recall that $T$ is a random variable which represents the rounds that $\+A$ runs. Assume $\Pr[{\@{N}_0}]{T\geq t^*}< 0.1$. Then we have $\Pr[{\@{N}_0}]{\ol E}< 0.075+0.1$ from the union bound. Combining \Cref{thm: huber} and \Cref{lem: gaussian KL}, we get
	\[
		\Pr[{\@{N}_{\!{mix}}}]{E}\geq \frac{m}{2\tp{m-1+\exp\tp{\frac{\eps^2t^*}{\sigma^2}}}}-\Pr[{\@{N}_0}]{\ol E}> \frac{m}{2\tp{m-1+m+1}}-0.1-0.075 \geq 0.075
	\]
	for every $m\geq 1$. This indicates the existence of some $j\in[m]$ such that $\Pr[\@{N}_j]{E}>0.075$, which is in contradiction to the promised success probability of $\+A$. Therefore $\+A$ satisfies
	\[
		\Pr[{\@{N}_0}]{T\geq t^*}\geq 0.1.
	\]
\end{proof}

\subsubsection{From Gaussian to Bernoulli}\label{ssubsec: gau2ber}

We then show a reduction from Gaussian arms to Bernoulli arms which implies lower bounds for instances $\set{\@{H}_j^{(m)}}_{j\in[m]\cup \set{0}}$. 

Given an input instance from $\set{\@{N}_j}_{j\in[m]\cup\set{0}}$, we can map it to a corresponding instance among $\set{\@{H}_j^{(m)}}_{j\in[m]\cup\set{0}}$ by the following rules.

In each round, if an arm receives a loss $\ell\in \bb R$, let
\begin{equation}
    \wh \ell=\begin{cases}
        0,\quad \mbox{if } & \ell<0;\\
        1,\quad \mbox{if } & \ell\geq 0.
    \end{cases}\label{eq: hat l}
\end{equation}

Obviously, losses drawn from Gaussian distribution $\+N(0, \sigma^2)$ are mapped to $\!{Ber}\tp{\frac{1}{2}}$ losses. For a biased Gaussian $\+N\tp{-\eps,\sigma^2}$, as \Cref{fig: gau-ber} shows, it holds that
\begin{align*}
    \Pr{\wh \ell< 0}&=\int_{-\infty}^{-\eps}\frac{1}{\sqrt{2\pi}\sigma}e^{-\frac{(x+\eps)^2}{2\sigma^2}}\dd{x} + \int_{-\eps}^{0}\frac{1}{\sqrt{2\pi}\sigma}e^{-\frac{(x+\eps)^2}{2\sigma^2}}\dd{x}\\
    &=\frac{1}{2}+\int_{-\eps}^{0}\frac{1}{\sqrt{2\pi}\sigma}e^{-\frac{(x+\eps)^2}{2\sigma^2}}\dd{x}.
\end{align*}

\begin{figure}[ht]
	\centering
\tikzset{every picture/.style={line width=0.75pt}} 
\begin{tikzpicture}[x=0.75pt,y=0.75pt,yscale=-1,xscale=1]


\draw [draw opacity=0, fill=gray!30]  (248.21,41) .. controls (265.79,41.37) and (269.38,46.69) .. (278.13,48.79) -- (278.13,159.74) -- (248.21,159.74) -- (248.21,41) -- cycle;

\draw  (103.69,159) -- (410.47,159)(278.14,30.67) -- (278.14,180.6) (403.47,154) -- (410.47,159) -- (403.47,164) (273.14,37.67) -- (278.14,30.67) -- (283.14,37.67)  ;
\draw    (248.22,40.26) .. controls (309.93,40.73) and (327.51,157.16) .. (396.77,157.16) ;
\draw    (248.22,40.26) .. controls (186.5,40.73) and (168.92,157.16) .. (99.67,157.16) ;
\draw  [dash pattern={on 0.84pt off 2.51pt}]  (248.22,40.26) -- (248,158.13) ;

\draw   (72,176.5) .. controls (72,181.17) and (74.33,183.5) .. (79,183.5) -- (162.5,183.5) .. controls (169.17,183.5) and (172.5,185.83) .. (172.5,190.5) .. controls (172.5,185.83) and (175.83,183.5) .. (182.5,183.5)(179.5,183.5) -- (266,183.5) .. controls (270.67,183.5) and (273,181.17) .. (273,176.5) ;
\draw   (283,177) .. controls (283.02,181.67) and (285.36,183.99) .. (290.03,183.97) -- (331.03,183.79) .. controls (337.7,183.76) and (341.04,186.08) .. (341.06,190.75) .. controls (341.04,186.08) and (344.36,183.73) .. (351.03,183.7)(348.03,183.72) -- (392.03,183.52) .. controls (396.7,183.5) and (399.02,181.16) .. (399,176.49) ;

\draw (265,159.72) node [anchor=north west][inner sep=0.75pt]   [align=left] {$0$};
\draw (240,162) node [anchor=north west][inner sep=0.75pt]   [align=left] {$-\eps$};
\draw (231,13.4) node [anchor=north west][inner sep=0.75pt]    {$\frac{1}{\sqrt{2\pi}\sigma}$};
\draw (158,192.4) node [anchor=north west][inner sep=0.75pt]    {$\hat \ell =0$};
\draw (327,192.4) node [anchor=north west][inner sep=0.75pt]    {$\hat \ell =1$};

\end{tikzpicture}
\caption{From Gaussian to Bernoulli}
\label{fig: gau-ber}
\end{figure}
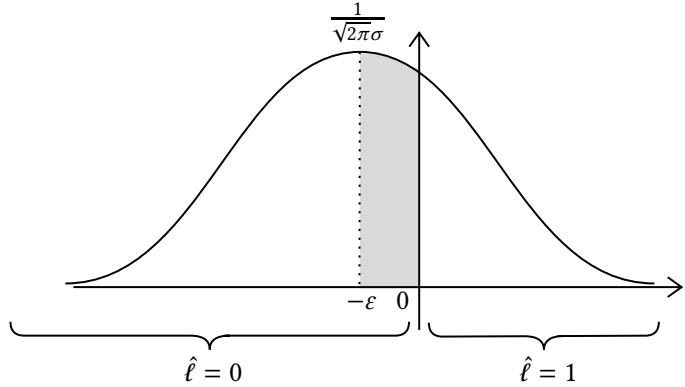

Let $f(\sigma)=\int_{-\eps}^{0}\frac{1}{\sqrt{2\pi}\sigma}e^{-\frac{(x+\eps)^2}{2\sigma^2}}\dd{x}$ denote the shadowed area in \Cref{fig: gau-ber}. Note that $f$ is continuous with regard to $\sigma$ and 
\[
    f(\sigma)\in\tp{\frac{\eps}{\sqrt{2\pi}\sigma}e^{-\frac{\eps^2}{2\sigma^2}}, \frac{\eps}{\sqrt{2\pi}\sigma}}.
\]
Assume that $\eps<\frac{1}{8}$. Therefore, there exists $\sigma_0\in\tp{\frac{1}{2\sqrt{2\pi}},\frac{1}{\sqrt{2\pi}}}$ such that $f(\sigma_0)=\eps$. Choose $\sigma=\sigma_0$. Then we map $\+N(-\eps,\sigma^2)$ to $\!{Ber}\tp{\frac{1}{2}-\eps}$ and transform the sample space from $\bb R^{m\times t^*}$ to $\set{0,1}^{m\times t^*}$.

\begin{lemma}\label{lem: non-adapt-BAI-lb}
	Let $\eps$ be a number in $\tp{0,\frac{1}{8}}$. For any algorithm distinguishing $\set{\@{H}^{(m)}_j}_{j\in[m]\cup\set{0}}$ with probability $0.925$, we have $\Pr[\@{H}^{(m)}_0]{T\geq t^*}\geq 0.1$.
\end{lemma}
\begin{proof}
	Assume that there exists such an algorithm $\+A$ with $\Pr[\@{H}^{(m)}_0]{T\geq t^*}< 0.1$. We then construct an algorithm $\+A'$ to distinguish $\set{\@{N}_j}_{j\in[m]\cup\set{0}}$.

	The algorithm $\+A'$ proceeds as follows: When $\+A'$ receives a loss $\ell$, it first calculates $\wh \ell$ as \Cref{eq: hat l} and treats $\wh \ell$ as the loss to apply $\+A$. If $\+A$ outputs $\@{H}^{(m)}_j$, $\+A'$ output $\@{N}_j$. Therefore, $\+A'$ also succeeds with probability $0.925$ while satisfying $\Pr[\@{N}_0]{T\geq t^*}< 0.1$. This violates \Cref{lem: gaussian-lb}.
\end{proof}

We remark that we cannot replace $\@{H}^{(m)}_0$ by $\@{H}^{(m)}_j$ for any $j\in [m]$ in \Cref{lem: non-adapt-BAI-lb}, since an ``$\@{H}^{(m)}_j$ favourite'' algorithm exists for every $j\in [m]$. For example, an ``$\@{H}^{(m)}_1$ favourite'' algorithm is as follows: one first sample the arms for $\frac{2\log \frac{1}{0.03}}{\eps^2}$ rounds. If the empirical mean $\wh p_1< \frac{1}{2}-\frac{\eps}{2}$, terminate and output $\@{H}^{(m)}_1$. Otherwise apply an algorithm which can distinguish $\set{\@{H}_j^{(m)}}_{j\in[m]\cup\set{0}}$ with probability $0.96$. By the Hoeffding's inequality, the error probability in the first stage is at most $0.03$. Therefore, this ``$\@{H}^{(m)}_1$ favourite'' algorithm has success probability $0.925$ and with high probability, it only needs to play $\frac{2\log \frac{1}{0.03}}{\eps^2}$ rounds when the input instance is $\@{H}^{(m)}_1$.

Then we are ready to prove \Cref{lem:block-BAI-lb}, which is a direct corollary of the following lemma.
\begin{lemma}\label{lem: block-BAI-lb-simple}
	Let $\eps$ be a number in $\tp{0,\frac{1}{8}}$ and assume $m\geq 2$. There exists a constant $c_1>0$ such that for any algorithm $\+A$ which can output an $\eps$-optimal arm on any instance among $\set{\@{H}^{(m)}_j}_{j\in[m]\cup\set{0}}$ with probability at least $0.95$, we have $\E[\@{H}^{(m)}_0]{T}\geq \frac{c_1\log (m+1)}{\eps^2}$.
\end{lemma}
\begin{proof}
	We first consider the case $c_0\log (m+1)> 4\log 40$ where $c_0$ is the universal constant in the definition of $t^*$. We reduce from the hypothesis testing lower bound in \Cref{lem: non-adapt-BAI-lb}. Assume $\+A$ satisfying $\Pr[\@{H}^{(m)}_0]{T\geq \frac{c_0\log (m+1)}{2\eps^2}}< 0.1$. Then we construct an algorithm $\+A'$ to distinguish $\set{\@{H}^{(m)}_j}_{j\in[m]\cup\set{0}}$. Given an instance among $\set{\@{H}^{(m)}_j}_{j\in[m]\cup\set{0}}$, we first apply $\+A$ to get an output arm $i$. Then we sample $\frac{2\log \frac{1}{0.025}}{\eps^2}$ rounds and check whether the empirical mean $\wh p_i\leq \frac{1}{2}-\frac{\eps}{2}$. If so, output $\@{H}^{(m)}_i$. Otherwise, output $\@{H}^{(m)}_0$. The success probability of at least $0.925$ is guaranteed by Hoeffding's inequality and the union bound. 
	
	According to our assumption, with probability larger than $0.9$, $\+A'$ terminates in $\frac{c_0\log (m+1)}{2\eps^2} + \frac{2\log \frac{1}{0.025}}{\eps^2}< \frac{c_0\log (m+1)}{\eps^2}$ rounds. This violates \Cref{lem: non-adapt-BAI-lb}.

	Then we consider the case $c_0\log (m+1)\leq 4\log 40$; that is, when $m$ is bounded by some constant. It then follows from \Cref{lem: non-adapt-lb-small-m} that $\+A$ satisfies $\Pr[\@{H}^{(m)}_0]{T\geq \frac{c_s}{\eps^2}}\geq 0.1$ for a universal constant $c_s$ when $m\geq 2$. 

	Then choosing $c_1=\min\set{\frac{c_0}{20}, \frac{c_s}{10\log (m_0+1)}}$ where $m_0=\lfloor e^{\frac{4\log 40}{c_0}}-1 \rfloor$, we have $\E[\@{H}^{(m)}_0]{T}\geq \frac{c_1\log (m+1)}{\eps^2}$ for any algorithms that can output an $\eps$-optimal arm on any instance among $\set{\@{H}^{(m)}_j}_{j\in[m]\cup\set{0}}$ with probability at least $0.95$ when $m\geq 2$.
\end{proof}

\subsection{The Lower Bound for \texorpdfstring{$\*m$}~-\BAI}\label{subsec: block BAI-lb}

Recall that in $\*m$-\BAI, the $N$ arms are partitioned into $K$ groups with size $m_1,m_2,\dots, m_K$ respectively. Each pull of an arm results in an observation of all the arms in its group. 
Consider an $\*m$-\BAI instance $\@{H}^{\*m}_0$ which consists of all fair coins. Recall that we use $T^{(k)}$ to denote the number of rounds in which the pulled arm belongs to the $k$-th group. 

We then prove the following lemma, which indicates the result of \Cref{thm:lb-m-BAI} directly.
\begin{lemma}\label{lem: block-BAI-lb}
	Let $\eps$ be a number in $\tp{0,\frac{1}{8}}$. For every $\tp{\eps,0.05}$-PAC algorithm of $\*m$-\BAI, we have $\E[\@{H}^{\*m}_0]{T^{(k)}}\geq \frac{c_1\log (m_k+1)}{\eps^2}$ for every $k\in[K]$ with $m_k\geq 2$ and $\E[\@{H}^{\*m}_0]{T}\geq \sum_{k=1}^K \frac{c_1\log (m_k+1)}{2\eps^2}$ if the total number of arms $\sum_{k=1}^K m_k\geq 2$,  where $c_1$ is the constant in \Cref{lem: block-BAI-lb-simple}.

	Moreover, these lower bounds still hold even the algorithm can identify the $\eps$-optimal arm with probability $0.95$ only when the input arms have losses drawn from either $\!{Ber}\tp{\frac{1}{2}}$ or $\!{Ber}\tp{\frac{1}{2}-\eps}$.
\end{lemma}
\begin{proof}
	We only prove the latter case which is stronger. Let $\+H$ be the set of all $\*m$-\BAI instances where the input arms have losses drawn from either $\!{Ber}\tp{\frac{1}{2}}$ or $\!{Ber}\tp{\frac{1}{2}-\eps}$. 

	Let $\+A$ be an algorithm that identifies the $\eps$-optimal arm with probability $0.95$ when the input instance is in $\+H$. Assume $\+A$ satisfies $\E[\@{H}^{\*m}_0]{T^{(k)}}<\frac{c_1\log (m_k+1)}{\eps^2}$ for some $k\in[K]$. In the following, we construct an algorithm $\+A'$ to find an $\eps$-optimal arm given instances in $\set{\@{H}^{(m_k)}_j}_{j\in[m]\cup\set{0}}$. 

	Given any $(m_k)$-\BAI instance $\@{H}^{(m_k)}\in \set{\@{H}^{(m_k)}_j}_{j\in[m]\cup\set{0}}$ , we construct an $\*m$-\BAI instance: set $\@{H}^{(m_k)}$ to be the $k$-th group and all remaining arms are fair ones. Then we apply $\+A$ on this instance.  
	The output of $\+A'$ is as follows:
	\[
		\mbox{Output of }\+A'=
		\begin{cases}
			\mbox{ arm }j, &\mbox{if the output of }\+A \mbox{ is arm }(k,j);\\
			\mbox{ an arbitrary arm}, &\mbox{otherwise}.
		\end{cases}
	\]

	Clearly, the correct probability of $\+A'$ is at least $0.95$. However, $\+A'$ satisfies $\E[\@{H}_0^{(m_k)}]{T}< \frac{c_1\log (m_k+1)}{\eps^2}$, which violates \Cref{lem: block-BAI-lb-simple}.
	
	Therefore, we have $\E[\@{H}^{\*m}_0]{T^{(k)}}\geq \frac{c_1\log (m_k+1)}{\eps^2}$ for every $k\in[K]$ with $m_k\geq 2$ and thus have proved $\E[\@{H}^{\*m}_0]{T}\geq \sum_{k=1}^K \frac{c_1\log (m_k+1)}{\eps^2}$ as long as each $m_k\geq 2$. For those groups of size one, we can pair and merge them so that each group contains at least two arms (in case there are odd number of singleton groups, we merge the remaining one to any other groups). Notice that this operation only makes the problem easier (since one can observe more arms in each round) and only affects the lower bound by a factor of at most $2$. Therefore, we still have
	\[
		\E[\@{H}^{\*m}_0]{T} \geq \sum_{k=1}^K \frac{c_1\log (m_k+1)}{2\eps^2}.
	\]
	
	

\end{proof}

\section{Regret Lower Bounds}\label{sec: regret lb}
In this section we prove lower bounds for minimax regrets in various settings. All lower bounds for regrets in the section are based on the lower bounds for $\*m$-\BAI established in \Cref{sec: BAI lb}.

\subsection{Regret Lower Bound for \texorpdfstring{$\*m$}~-\MAB}\label{subsec: regret-lb-MAB}

Let us fix $\*m=(m_1,\dots,m_K)$. We then derive a regret lower bound for $\*m$-MAB and thus prove \Cref{thm:lb-m-MAB}.
Let $T$ be the time horizon and $c_1$ be the constant in \Cref{lem: block-BAI-lb-simple}. Consider a set of $\*m$-BAI instances where each arm has losses drawn from either $\!{Ber}\tp{\frac{1}{2}}$ or $\!{Ber}\tp{\frac{1}{2}-\eps}$ where $\eps=\sqrt{\frac{c_1\sum_{k=1}^K \log (m_k+1)}{8T}}$. Denote this set by $\+H$.
\begin{lemma}\label{lem:lb-m-MAB}
	For any algorithm $\+A$ of $(m_1,\dots,m_k)$-\MAB, for any sufficiently large $T>0$, there exists $\@H\in \+H$ such that the expected regret of $\+A$ satisfies
	\[
		\E[\@H]{R(T)}\geq c'\cdot \sqrt{T\cdot\sum_{k=1}^K \log (m_k+1)}
	\]
	where $c'>0$ is a universal constant. Here the expectation is taken over the randomness of losses which are drawn from $\@H$ independently in each round.
\end{lemma}
\begin{proof}
	Assume $\+A$ satisfies
	\[
		\E[\@H]{R(T)} < \frac{\sqrt{T\cdot \frac{1}{2}\sum_{k=1}^K c_1\log (m_k+1) }}{5000}
	\]
	for every $\@H\in \+H$ where $c_1$ is the constant in \Cref{lem: block-BAI-lb-simple}. \Cref{lem: reduction} shows that $\+A$ implies an algorithm to identify the $\eps$-optimal arm for $\*m$-\BAI instances in $\+H$ with probability $0.95$ which terminates in $c_1\cdot \frac{\sum_{k=1}^K \log (m_k+1)}{8\eps^2}$ rounds. We can assume $\eps<\frac{1}{8}$ since $T$ is sufficiently large.

	However, according to \Cref{lem: block-BAI-lb}, for any such algorithms, there exists some instances in $\+H$ that need at least $\frac{c_1\sum_{k=1}^K \log (m_k+1)}{2\eps^2}$ rounds. This violates \Cref{lem: reduction} and thus indicates a regret lower bound of $\Omega\tp{\sqrt{T\cdot\sum_{k=1}^K \log (m_k+1)}}$.
\end{proof}

\Cref{thm:lb-m-MAB} is a direct corollary of \Cref{lem:lb-m-MAB}.




\subsection{Regret Lower Bounds for Strongly Observable Graphs}\label{subsec: regret-lb-strongly}
Let $G=(V,E)$ be a strongly observable graph with a self-loop on each vertex. Let $N=\abs{V}$. Assume that there exist $K$ \emph{disjoint} sets $S_1,\dots ,S_K\subseteq V$ such that there is no edge between $S_i$ and $S_j$ for any $i\ne j$.
For every $k\in [K]$, let $m_k=\abs{S_k}$. Let $S=\bigcup_{k\in[K]} S_k$. 


\begin{proof}[Proof of \Cref{thm:lb-graph-bandit-strongly}]
	We present a reduction from $\*m$-\MAB to bandit with feedback graph $G$ where $\*m=(m_1,\dots,m_K)$. Let $\+A$ be an algorithm for bandit with feedback graph $G$. Consider a set of instances where the loss of each arm is drawn from either $\!{Ber}\tp{\frac{1}{2}}$ or $\!{Ber}\tp{\frac{1}{2}-\eps}$ where $\eps=\sqrt{\frac{c_1\sum_{k=1}^K \log (m_k+1)}{8T}}$ (here $c_1$ is the constant in \Cref{lem: block-BAI-lb-simple}). Denote this set by $\+H$. When we say the input of \MAB is an instance in $\+H$, we mean that the loss sequence is drawn from this instance independently in each round.

	Then we design an algorithm $\+A'$ for $\*m$-\MAB to deal with instances in $\+H$ as follows. For an $\*m$-\MAB instance $\@H^{\*m}$ in $\+H$, we construct a bandit instance with feedback graph $G$: the losses of arms in $S_k$ correspond to the losses of arms in the $k$-th group of $\@H^{\*m}$ in the $\*m$-\MAB game and the losses of arms in $V\setminus S$ are always equal to $1$. 

	The algorithm $\+A'$ actually makes decisions according to $\+A$. If $\+A$ pulls an arm in $S$, $\+A'$ pulls the corresponding arm in the $\*m$-\MAB game. Otherwise, when $\+A$ requests to pull an arm $A_t\in V\setminus S$, we replace this action by letting $\+A'$ pull the first arm in each group once and then feed the information that $A_t$ should have observed back to $\+A$ (Note that all arms outside $S$ have fixed loss $1$). We force $\+A'$ to terminate after pulling exactly $T$ arms.
	Note that $\eps\ll \frac{1}{K}$ since $T$ is sufficiently large.  If we use $R(T)$ and $R'(T)$ to denote the regret of $\+A$ and $\+A'$ respectively, then by our choice of $\eps$, we have
	\[
		\E{R(T)}\ge \E{R'(T)}
	\]
	where the expectation is taken over the randomness of loss sequences specified above.


	\Cref{lem:lb-m-MAB} shows that there exists $\@H\in \+H$ such that
	\[
		\E[\@H]{R'(T)}\ge c'\sqrt{T\cdot \sum_{k=1}^K \log (m_k+1)}
	\]
	Therefore, there exist some loss sequences on which $\+A$ needs to suffer a regret of $\Omega\tp{\sqrt{T\cdot \sum_{k=1}^K \log (m_k+1)}}$.
\end{proof}

\begin{remark}
	Although we assume each vertex has a self-loop in \Cref{thm:lb-graph-bandit-strongly}, it is easy to verify that this result also holds for strongly observable graphs which contain some vertices without self-loops, as long as we can find legal $\set{S_k}_{k\in[K]}$. For example, for the loopless clique, we can also apply \Cref{thm:lb-graph-bandit-strongly} with $K=1$ and $S_1=V$. It gives a minimax regret lower bound of $\Omega\tp{\sqrt{T\log N}}$, which matches the previous best upper bound in~\cite{ACDK15}.
\end{remark}

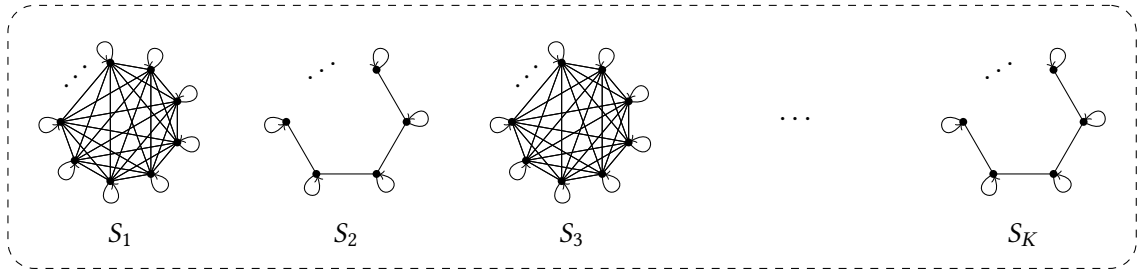
\begin{figure}[ht]
	\centering
        \begin{tikzpicture}
            \tikzset{mynode/.style=fill, circle, inner sep=1pt, minimum size=3pt}
			\node[draw,dashed,rectangle,rounded corners=10pt,minimum width=15cm,minimum height=3.5cm] (box) at (6,2.8) {};
			\foreach \k in {0, 6}{
                \foreach \j in {0,1,2,...,7}{
                    \pgfmathparse{cos(40*(\j))};
                    \tikzmath{\x =\pgfmathresult;}
                    \pgfmathparse{sin(40*(\j))};
                    \tikzmath{\y =\pgfmathresult;}
                    \node[mynode] at (\k-0.8*\x,3-0.8*\y){};
					\coordinate (MAB\j) at (\k-0.8*\x,3-0.8*\y){};
                }
                \pgfmathparse{cos(40*8)};
                \tikzmath{\x =\pgfmathresult;}
                \pgfmathparse{sin(40*8)};
                \tikzmath{\y =\pgfmathresult;}
                \node[color=black] at (\k-0.6*\x,3-0.6*\y){\begin{rotate}{85}$\ddots$\end{rotate}};
                \foreach \j in {0,1,2,3,4,5,6,7}{
					\path[-] (MAB\j) edge [in=\j*40-130, out=\j*40-210,loop] ();
                }
                \foreach \i in {0,1,2,3,4,5,6,7}{
                	\foreach \j in {0,1,2,3,4,5,6,7}{
                    	\path[-] (MAB\i) edge (MAB\j);
                	}
                }
			}
			\node () at (0, 1.5){$S_1$};
			\node () at (3, 1.5){$S_2$};
			\node () at (6, 1.5){$S_3$};
			\node () at (12, 1.5){$S_K$};
			\foreach \k in {3, 12}{
                \foreach \j in {0,1,2,3,4}{
                    \pgfmathparse{cos(60*(\j))};
                    \tikzmath{\x =\pgfmathresult;}
                    \pgfmathparse{sin(60*(\j))};
                    \tikzmath{\y =\pgfmathresult;}
                    \node[mynode] at (\k-0.8*\x,3-0.8*\y){};
					\coordinate (MAB\j) at (\k-0.8*\x,3-0.8*\y){};
                }
                \pgfmathparse{cos(60*5)};
                \tikzmath{\x =\pgfmathresult;}
                \pgfmathparse{sin(60*5)};
                \tikzmath{\y =\pgfmathresult;}
                \node[color=black] at (\k-0.5*\x,3-0.5*\y){\begin{rotate}{70}$\ddots$\end{rotate}};
                \foreach \j in {0,1,2,3,4}{
					\path[-] (MAB\j) edge [in=\j*60-130, out=\j*60-210,loop] ();
                }
                \foreach \i in {0,1,2,3}{
					\pgfmathparse{\i+1};
					\tikzmath{\j =\pgfmathresult;}
                    \path[-] (MAB\i) edge (MAB\j);
                }
			}
			\node[color=black] at (9,3){$\cdots$};
        \end{tikzpicture}
		\caption{A Feedback Graph Example}
		\label{fig: graph instance}
\end{figure}

\Cref{thm:lb-graph-bandit-strongly} gives a general regret lower bound for bandit with arbitrary feedback graphs. Intuitively, it allows us to partition the graph and consider the hardness of each single part respectively.

For example, consider the graph shown in \Cref{fig: graph instance}: The feedback graph is the disjoint union of $K_1$ cliques and $K_2=K-K_1$ cycles where each clique contains $m_1$ vertices and each cycle contains $m_2$ vertices. Note that the clique cover of this graph contains $K_1$ cliques of size $m_1$ and $\lceil\frac{K_2m_2}{2}\rceil$ cliques of constant size. According to \Cref{thm:ub_clique}, our Algorithm~\ref{algo:main-algo} gives a regret upper bound of $O\tp{\sqrt{T\tp{K_1\log m_1 + K_2m_2}}}$, which matches the lower bound given in \Cref{thm:lb-graph-bandit-strongly}. The previous best lower bound~(\cite{ACDK15}) on this feedback graph is $\Omega\tp{\sqrt{\tp{K_1+K_2m_2}T}}$. When $K_1$ and $m_1$ are large, our result wins by a factor of $\Theta\tp{\sqrt{\log m_1}}$.

\subsection{Regret Lower Bounds for Weakly Observable Graphs}\label{sec: weakly-regret-lb}

Let $G=(V,E)$ be a weakly observable graph. Assume that $V$ can be partitioned into $K$ disjoint sets $V=V_1\cup V_2\cup \cdots\cup V_K$ and each $G[V_k]$ contains a $t_k$-packing independent set $S_k$ such that every vertex in $S_k$ does not have a self-loop. Assume there are no edges from $V_j$ to $S_i$ for any $i\neq j$. Let $m_k=\abs{S_k}$ and $S=\bigcup_{k\in[K]}S_k$. 

Without loss of generality, we assume in the following proof that each $m_k\geq 2$. When there exists some $m_k=1$, we can pair and merge them into new sets of size  at least $2$ (in case there are odd number of singleton sets, we merge the remaining one to any other sets). This merging process only affects the result by at most a constant factor. Let $\*m=(m_1,\dots,m_K)$. Our proof idea is to embed a certain $\*m'$-\BAI instance in $G$ so that the lower bound follows from the lower bound of $\*m'$-\BAI.

\begin{proof}[Proof of \Cref{thm:lb-graph-bandit-weakly}]
	Let 
	\[
		\xi_k=\max\set{c_1\log (m_k+1), \frac{c_2m_k}{t_k}}
	\] 
	for every $k\in [K]$ where $c_1>0$ is the constant in \Cref{lem: block-BAI-lb} and $c_2=\frac{c_1\log 3}{4}$. Assume there exists an algorithm $\+A$ such that 
	\begin{equation}
		R(T)<\frac{1}{2\cdot 1250^\frac{2}{3}}\tp{\sum_{k=1}^K \xi_k}^{\frac{1}{3}}\cdot T^{\frac{2}{3}}\label{eq: weakly R}
	\end{equation}
	for every loss sequence. We will construct an $\*m'$-\BAI game for some $\*m'=\tp{m'_1,m'_2,\dots, m'_{K'}}$ and reduce this \BAI game to the bandit problem with feedback graph $G$. The vector $\*m'$ is obtained from $\*m$ in the following ways. For every $k\in [K]$, we distinguish between two cases:
	
	\begin{itemize}
		\item Case $1$: if $c_1\log (m_k+1) \geq \frac{c_2m_k}{t_k}$, we let the arms in $S_k$ form a group in the $\*m'$-\BAI instance;
		\item Case $2$: if $c_1\log (m_k+1) < \frac{c_2m_k}{t_k}$, we divide $S_k$ into $\lfloor \frac{m_k}{2} \rfloor$ small sets, each with size at least two. Each small set becomes a group in the $\*m'$-\BAI instance.
	\end{itemize}

	In other words, each group in the $\*m'$-\BAI instance is either one of $S_k$ (Case 1) or is a subset of a certain $S_k$ (Case 2).
	
	Given an $\*m'$-\BAI instance and time horizon $T>0$, we now define the loss sequence for bandit with feedback graph $G$: the losses of arms in $S$ in each round are sampled from the distribution of the corresponding arm in the $\*m'$-\MAB instance independently, and the losses of arms in $V\setminus S$ are always equal to $1$. We then design an algorithm $\+A'$ for the $\*m'$-\BAI game by simulating $\+A$ on this graph bandit problem. 
	If $\+A$ pulls an arm in $V\setminus S$ and observes arms in $S_k$, we again consider two cases:
	\begin{itemize}
		\item Case 1: if $c_1\log (m_k+1) \geq \frac{c_2m_k}{t_k}$, we let $\+A'$ pull an arbitrary arm in the corresponding group $\*m'$-\MAB instance;
		\item Case 2: if $c_1\log (m_k+1) < \frac{c_2m_k}{t_k}$, for each arm in $S_k$ that will be observed, $\+A'$ pulls the corresponding arm in the  $\*m'$-\MAB instance once.
	\end{itemize}
	Otherwise if $\+A$ pulls an arm in $S$, $\+A'$ does nothing and just skips this round. Note that $\+A'$ can always observe more information about the feedback of arms in $S$ than $\+A$. So $\+A'$ can well simulate $\+A$ just by feeding the information it observed to $\+A$ and making decisions according to the behavior of $\+A$ as described above.
	
	Let $T_i$ be the number of times that arm $i$ has been pulled by $\+A$. At the end of the game, $\+A'$ samples an arm in $V$ according to the distribution $\tp{\frac{T_1}{T}, \frac{T_2}{T},\dots, \frac{T_N}{T}}$. If the sampled arm is in $V\setminus S$, $\+A'$ outputs a random arm. Otherwise $\+A'$ outputs the sampled arm. Choose $\eps =1250^{\frac{1}{3}} \tp{\frac{\sum_{k=1}^K \xi_k}{T}}^{\frac{1}{3}}$. We can verify that $\+A'$ is an $\tp{\eps, 0.05}$-PAC algorithm through an argument similar to the one in our proof of \Cref{lem: reduction}.

	Let $T^{(k)}$ be the number of times that the arms in group $k$ have been pulled by $\+A'$ in the $\*m'$-\BAI game. According to \Cref{lem: block-BAI-lb}, for each $k\in[K']$,
	\[
		\E[\@{H}^{\*m'}_0]{T^{(k)}}\geq  \frac{c_1\log (m'_k+1)}{\eps^2},
	\] 
	where $\@{H}^{\*m'}_0$ is the $\*m'$-\BAI instance with all fair coins. Let $\@I_0$ denote the graph bandit instance constructed from above rules based on $\@{H}^{\*m'}_0$. Recall that one pull of $\+A$ corresponds to at most $t_k$ pulls of $\+A'$ in Case $2$. Therefore, when the input is $\@I_0$, $\+A$ must pull the arms in $V_k\setminus S_k$ for at least $\frac{c_1\lfloor \frac{m_k}{2} \rfloor\log 3}{t_k\eps^2}\geq \frac{c_2m_k}{t_k\eps^2}$ times if $k$ is in Case $2$ and at least $\frac{c_1\log (m_k+1)}{\eps^2}$ times if $k$ is in Case $1$. In other words, $\+A$ must pull the arms in $V_k\setminus S_k$ for at least $\frac{\xi_k}{\eps^2}$ times for every $k\in [K]$. Plugging in our choice of $\eps$, $\+A$ needs to pull the arms in $V\setminus S$ for more than $\frac{1}{1250^\frac{2}{3}}\cdot \tp{\sum_{k=1}^K \xi_k}^{\frac{1}{3}} T^{\frac{2}{3}}$ times in total on $\@I_0$. These pulls contribute a regret of at least $\frac{1}{2\cdot 1250^\frac{2}{3}}\tp{\sum_{k=1}^K \xi_k}^{\frac{1}{3}}\cdot T^{\frac{2}{3}}$, which contradicts the assumption in \Cref{eq: weakly R}.

	Therefore, there exists some loss sequences such that $\+A$ satisfies
	\[
		R(T)=\Omega\tp{T^{\frac{2}{3}}\cdot \tp{\sum_{k=1}^K\max\set{\log {m_k}, \frac{m_k}{t_k}}}^{\frac{1}{3}}}.
	\]
\end{proof}

\Cref{thm:lb-graph-bandit-weakly} confirms a conjecture in~\cite{HZ22}. It can also generalize the previous lower bound for weakly observable graphs $\Omega\tp{T^{\frac{2}{3}}\tp{\log \abs{S},\frac{\abs{S}}{t}}^{\frac{1}{3}}}$ in~\cite{CHLZ21} by applying \Cref{thm:lb-graph-bandit-weakly} with $K=1$ and $V_1=V$ where $S\subseteq V$ is a $t$-packing independent set of $G$. As consequences, \Cref{thm:lb-graph-bandit-weakly} provides tight lower bounds for several feedback graphs. For example, when $G$ is the disjoint union of $K$ complete bipartite graphs of size $m_1,m_2,\dots, m_K$ respectively, it implies a lower bound of $\Omega\tp{\tp{\sum_{k\in[K]}\log m_k}^{\frac{1}{3}}T^{\frac{2}{3}}}$, which matches the upper bound in~\cite{HZ22}.

\bibliography{block_bandit}
\bibliographystyle{alpha}


\appendix

\section{Lower Bound for \texorpdfstring{$(m)$}~-BAI with Bounded \texorpdfstring{$m$}~}\label{sec: non-adapt-lb-small-m}
In this section, we will lower bound the number of pulls in $\tp{\eps,0.05}$-PAC algorithms of $(m)$-BAI when $m$ is bounded by a constant. To this end, we first prove a likelihood lemma in \Cref{subsec: likelihood}.
\subsection{Likelihood Lemma}\label{subsec: likelihood}
Consider two instances $\@{H}_a$ and $\@{H}_b$ which only differ at one arm (without loss of generality, assume it is the first arm). In $\@{H}_a$, $\ell(1)$ is drawn from $\!{Ber}\tp{\frac{1}{2}}$ and in $\@{H}_b$, $\ell(1)$ is drawn from $\!{Ber}\tp{\frac{1}{2}-\eps}$ where $\eps\in\tp{0,\frac{1}{2}}$ is a fixed number.

Let $\+A$ be a PAC algorithm for \BAI. Let $K_j^t=\sum_{r=1}^t \ell^{(r)}(j)$ be the accumulative loss of arm $j$ before the $(t+1)$-th round and abbreviate $K_j^{N_j}$ as $K_j$. Let $A_j$ be the event that $N_j<\hat{t}$ for a fixed $\hat{t}\in\bb N$. Let $C^a_j$ be the event that $\set{\max_{1\leq t\leq \hat{t}} \abs{K_j^t-\frac{1}{2}t} < \sqrt{\hat{t}\cdot c\eps^2\hat{t}}}$ and $C_j^b$ be the event $\set{\max_{1\leq t\leq \hat{t}} \abs{K_j^t-\tp{\frac{1}{2}-\eps}t} < \sqrt{\hat{t}\cdot c\eps^2\hat{t}}}$ where $c$ is a positive constant.

\begin{lemma}[Lemma 3 of~\cite{MT04}]\label{lem: exp-ineq}
	If $0\leq x\leq \frac{1}{\sqrt{2}}$ and $y>0$, then $(1-x)^y\geq e^{-dxy}$ where $d=1.78$.
\end{lemma}

\begin{lemma}[Likelihood Lemma]\label{lem: likelihood}
	Let $S^a=A_1\cap B\cap C_1^a$ and $S^b=A_1\cap B\cap C_1^b$ where $B$ is an arbitrary event. Then we have 
	\begin{equation}
		\Pr[\@{H}_b]{S^a}\geq e^{-8(1+\sqrt{c})\eps^2\hat{t}}\Pr[\@{H}_a]{S^a}
		\label{eq: likeli1}
	\end{equation}
	and 
	\begin{equation}
		\Pr[\@{H}_a]{S^b}\geq e^{- 8(1+\sqrt{c})\eps^2\hat{t}}\Pr[\@{H}_b]{S^b}
		\label{eq: likeli2}
	\end{equation}
\end{lemma}
\begin{proof}
	We first prove \Cref{eq: likeli1}. For each $\omega\in S^a$ ($\omega$ is a history of the algorithm, including the behavior of the algorithm and observed result in each round), we have
	\begin{align*}
		\frac{\Pr[\@{H}_b]{\omega}}{\Pr[\@{H}_a]{\omega}}&=\frac{\tp{\frac{1}{2}-\eps}^{K_1}\tp{\frac{1}{2}+\eps}^{N_1-K_1}}{\tp{\frac{1}{2}}^{N_1}}=\tp{1-2\eps}^{K_1}\tp{1+2\eps}^{N_1-K_1}\\
		&=\tp{1-4\eps^2}^{N_1-K_1}\tp{1-2\eps}^{2K_1-N_1}\geq \tp{1-4\eps^2}^{N_1}\tp{1-2\eps}^{2K_1-N_1}.
	\end{align*}
	From \Cref{lem: exp-ineq} and the definition of $S^a$, we have
	\[
		\tp{1-4\eps^2}^{N_1}\geq \tp{1-4\eps^2}^{\hat{t}}\geq e^{-8\eps^2\hat{t}}
	\] and
	\[
		\tp{1-2\eps}^{2K_1-N_1}\geq \tp{1-2\eps}^{2\sqrt{\hat{t}\cdot c\eps^2\hat{t}}}\geq e^{-8\sqrt{c}\eps^2\hat{t}}.
	\]
	Therefore 
	\[
		\frac{\Pr[\@{H}_b]{\omega}}{\Pr[\@{H}_a]{\omega}}\geq e^{-8(1+\sqrt{c})\eps^2\hat{t}}
	\]
	and thus
	\[
		\Pr[\@{H}_b]{S^a}\geq \sum_{\omega\in S^a}\frac{\Pr[\@{H}_b]{\omega}}{\Pr[\@{H}_a]{\omega}}\cdot {\Pr[\@{H}_a]{\omega}} \geq e^{-8(1+\sqrt{c})\eps^2\hat{t}}\Pr[\@{H}_a]{S^a}.
	\]
	The proof of \Cref{eq: likeli2} is similar.
\end{proof}


\subsection{Lower Bound for \texorpdfstring{$(m)$}~-\BAI with Constant \texorpdfstring{$m$}~}
\begin{lemma}\label{lem: non-adapt-lb-small-m}
	There exists a constant $c_s$ such that for any algorithm $\+A$ which can output an $\eps$-optimal arm on any instance among $\set{\@{H}^{(m)}_j}_{j\in[m]\cup\set{0}}$ with probability at least $0.95$ when $m\geq 2$ and $c_0\log (m+1)\leq 4\log 40$, we have $\Pr[\@{H}^{(m)}_0]{T\geq \frac{c_s}{\eps^2}}\geq 0.1$.
\end{lemma}
\begin{proof}
	Note that there must exist $j\in[m]$ such that $\Pr[\@{H}^{(m)}_0]{\+A \mbox{ output arm }j}\leq \frac{1}{m}$. Let $B$ be the event that the algorithm output any arm except for arm $j$. Apply \Cref{lem: likelihood} with $\hat t = \frac{\log 3}{100\eps^2}$, $c=100$, $\@{H}_b = \@{H}^{(m)}_j$ and $\@{H}_a = \@{H}^{(m)}_0$. Assume that $\Pr[\@{H}^{(m)}_0]{T\geq \hat t}< 0.1$. By the Kolmogorov's inequality, we have $\Pr[\@{H}^{(m)}_0]{\max_{1\leq t\leq \hat{t}} \abs{K_j^t-\frac{1}{2}t} < \sqrt{\hat{t}\cdot c\eps^2\hat{t}}} \geq 1- 0.25$. Therefore, we have $\Pr[\@{H}^{(m)}_0]{S^a}\geq 0.9-\frac{1}{m}-0.25\geq 0.15$ by the union bound.

	Then from \Cref{eq: likeli1}, we have
	\[
		\Pr[\@{H}^{(m)}_j]{B}\geq e^{-8(1+\sqrt{c})\cdot \frac{\log 3}{100}} \cdot \Pr[\@{H}^{(m)}_0]{S^a}> 0.15\cdot \frac{1}{3}=0.05.
	\] 
	However, this is in contradiction with the success probability of $\+A$. Therefore, letting $c_s=\frac{\log 3}{100}$, we have $\Pr[\@{H}^{(m)}_0]{T\geq \frac{c_s}{\eps^2}}\geq 0.1$.
\end{proof}

\end{document}